\documentclass[11pt]{article}

\usepackage[final]{acl}

\usepackage{times}
\usepackage{latexsym}

\usepackage[utf8]{inputenc} 
\usepackage[T1]{fontenc}    
\usepackage{booktabs}       
\usepackage{graphicx}       
\usepackage{amsfonts}       
\usepackage{nicefrac}       
\usepackage{microtype}      
\usepackage{xcolor}         
\usepackage{amsmath}
\usepackage{amsthm}       
\usepackage{algorithm}     
\usepackage{algpseudocode} 
\usepackage{multirow}      
\usepackage{enumitem}      
\usepackage{subcaption}    
\usepackage{cleveref}

\usepackage{todonotes}
\usepackage{fancyvrb}
\usepackage{fvextra}
\usepackage{placeins}
\usepackage{thmtools}
\newtheorem{proposition}{Proposition}
\newtheorem{lemma}{Lemma}

\newtheorem*{remark}{Remark}
\newtheorem{definition}{Definition}

\newlength{\itemizelength}
\title{Unifying Conformal Language Tasks with In-Context Ensembles}

\author{
  \textbf{Xiao Shi Huang\textsuperscript{1}} \quad
  \textbf{Chen-Yuan Lin\textsuperscript{1}} \quad
  \textbf{Bruce Kuwahara\textsuperscript{1}} \\
  \textbf{Kin Kwan Leung\textsuperscript{2}} \quad
  \textbf{Jesse C. Cresswell\textsuperscript{2}} \\
  \textsuperscript{1}Signal 1 AI \quad  \textsuperscript{2}Layer 6 AI\\
  \small{\texttt{\{gary, thomas, bruce\}@signal1.ai}} \quad
  \small{\texttt{\{kk, jesse\}@layer6.ai}}
}

\begin{document}

\maketitle

\begin{abstract}
Many NLP tasks, such as summarization and extractive question answering, reduce to retrieving relevant content from documents under two constraints: \emph{coverage}, retaining enough pertinent information to achieve some goal, and \emph{conciseness}, removing as much irrelevant information as possible. Conformal prediction methods have been used to guarantee coverage, and must be optimized for conciseness through design of a score function. State-of-the-art scoring functions use hand-engineered LLM prompts asking the model to rate the importance of content, but manual prompt engineering is labor-intensive and task-specific. We introduce the \emph{Conformal Relevance} framework which uses in-context learning example curation and ensembling to create a score function which maintains coverage while improving conciseness with minimal manual input. We demonstrate this framework's application on seven NLP tasks, and also theoretically study the impact of diversity for ensembled conformal scores, giving a complementarity condition that characterizes when ensembling improves worst-case sentence scores, and a saturation bound on ensemble improvement.

\end{abstract}

\section{Introduction}
\label{sec:intro}

Many NLP tasks, including summarization~\citep{mukherjee-etal-2022-ectsum}, extractive question-answering (QA)~\citep{yang2018hotpotqa}, legal review~\citep{koreeda2021contractnli}, and clinical evidence selection~\citep{deyoung2020evidence}, reduce to retrieving relevant content from documents. Systems for these tasks must satisfy two demands: \emph{coverage} (relevant content is retained) and \emph{conciseness} (irrelevant content is excluded). Conformal prediction~\citep{vovk2005algorithmic} has become popular as a general purpose framework for providing distribution-free, finite sample coverage guarantees. It operates by calibrating an arbitrary scoring function over a labeled dataset. While coverage is guaranteed, conciseness strongly depends on the scoring function's predictive power.

Several examples of this framework have been studied in detail, differing in their relevance criterion. For question-answering, \citet{mohri2024conformal} define relevant content as non-hallucinated claims in the answer, while hallucinations should be filtered out. \citet{kuwahara2025conformal} look at extractive summarization, defining relevant content as important sentences, and filter out unimportant sentences for conciseness. For these NLP tasks and others, the scoring function is typically a large language model (LLM) configured through task-specific prompt engineering, which is labor-intensive, brittle, and not scalable~\citep{lu2022fantastically, min2022rethinking}.

We aim to subsume all such content-selection tasks under a general purpose relevance scoring function that replaces manual prompt-writing with in-context learning (ICL)~\citep{brown2020language, dong2023survey} example curation. Rather than \emph{describing} a relevance criterion in natural language as done in past work, we \emph{demonstrate} it through curated examples and let the LLM infer the criterion from the demonstrations.  We call this framework \emph{Conformal Relevance}: recall-oriented conformal calibration paired with a relevance scoring function, instantiated through an ICL-driven LLM score.

To address the diversity of relevance criteria in a unified way, we develop multiple ICL example selection strategies that induce LLM scoring functions with systematically different strengths and failure modes. Ensembling these scores improves performance while remaining entirely task-agnostic. Concretely, we average $K$ mechanistically distinct ICL-based scoring functions---chosen to span different signal types so that their failures are unlikely to coincide---and treat the result as a single conformal scoring function~\citep{ochoa2025conformal, waldron2026caos}. We mathematically formalize a \emph{complementarity condition} that characterizes when ensembling improves worst-case conformal scores. The marginal gain from each additional sub-scoring function gives diminishing returns, bounded by \(O(1/K)\).

We apply a fixed ICL scoring configuration on seven NLP datasets spanning five domains and show performance improvements over manually crafted prompts that define the relevance criterion per task. Our Conformal Relevance method removes substantially more irrelevant content at the same coverage on all seven tasks, with up to a ${\sim}50\%$ reduction in retained length. Through several ablations and controls, we determine that the gain stems from retrieval-induced diversity across the ensemble's in-context demonstrations. The required label budget is modest: 150--440 labels per task, comprising a shared ICL pool plus a 100-sample calibration set. Our contributions are:
\begin{enumerate}[noitemsep,topsep=0pt,leftmargin=*]
  \item A universal ICL ensemble scoring function for disparate content selection tasks with conformal coverage guarantees.   \item Four mechanistically diverse sub-scoring functions spanning distinct signal types. \item A mathematical formalization of score ensembling for recall-oriented conformal prediction.
\end{enumerate}

\vspace{-2pt}
\section{Background \& Related Work}
\vspace{-2pt}
\label{sec:background}

\label{sec:bg-conformal}

\textbf{Content selection.} Content selection is a generic task in NLP where relevant information must be extracted from a document. It occurs under many instantiations depending on the definition of relevance. The prototype is extractive summarization where relevance means a content span of the document is important~\citep{mukherjee-etal-2022-ectsum}. Extractive QA requires a model to select content from a corpus that is relevant to a question (retrieval) before using it to generate an answer~\citep{yang2018hotpotqa}, for example as done in Retrieval-Augmented Generation (RAG)~\citep {lewis2020rag}. PII detection~\citep{pilan2022text, shen2026pii, ponomarenko2026capid} (similar to named entity recognition~\citep{wang2025gptner}) selects content that constitutes private information so that it can be appropriately masked. These examples are far from comprehensive as content selection appears in a vast array of forms. Our aim is to unify content selection tasks throughout language modeling under a single framework for providing conformal guarantees on the capture of relevant content.

\textbf{Conformal prediction.} Given inputs $x\in \mathcal{X}$ and ground-truth values $y^*\in \mathcal{Y}$ drawn jointly from a distribution $(x, y^*)\sim \mathbb{P}$, split conformal prediction~\citep{vovk2005algorithmic, shafer2008tutorial} constructs a prediction set $C_{\hat q}(x_\text{test})\subseteq \mathcal{Y}$ for a new datapoint $x_\text{test}$ that come with a finite-sample coverage guarantee
\vspace{-2pt}
\begin{equation}\label{eq:coverage-guarantee}
     \mathbb{P}[y_\text{test}^* \in C_{\hat q}(x_\text{test})] \geq 1 - \alpha,
     \vspace{-2pt}
\end{equation}
where the error rate $\alpha$ is user-defined. Conformal prediction requires an arbitrary score function $S(x, y)$, which is often derived from a black-box machine learning model. By computing scores on a labeled calibration dataset, the conformal threshold $\hat q$ is set as the $\lceil(1-\alpha)(n+1)\rceil / n$ quantile of the $n$ scores. Then, prediction sets are generated as
\vspace{-2pt}
\begin{equation}\label{eq:prediction-set}
    C_{\hat q}(x_\text{test})  = \{y\in \mathcal{Y} \mid S(x_\text{test}, y) < \hat q \}.
    \vspace{-2pt}
\end{equation}
Smaller sets are preferred, \emph{ceteris paribus}. Conformal prediction is extremely useful because it makes no assumptions about the nature of scores, only that $x_\text{test}$ is exchangeable with the calibration data, a mild assumption that holds for IID settings \cite{angelopoulos2023conformal}.

\textbf{Conformal methods for content selection.} 
Conformal factuality~\citep{mohri2024conformal} studies a QA content selection setting where claims in an existing answer are filtered out until no hallucinations remain with high confidence. This is a \emph{precision}-style coverage guarantee where selected content must all fit the relevance criteria (here being non-hallucinated). Conformal factuality was further studied in RAG settings \citep{feng2025,chakraborty2026rag}. Extending from the QA setting, conformal importance \citep{kuwahara2025conformal} gives a \emph{recall}-style coverage guarantee for extractive summarization where the prediction set must select all relevant content with high confidence. Since recall-style coverage is more widely applicable to NLP tasks, we adopt it and recount conformal importance in more detail below.

Given a document $x$ with content spans $c_1,\ldots,c_m$ and a ground-truth important subset $y^* \subseteq \{c_1,\ldots,c_m\}$, a relevance function $R(c;x) \in [0,1]$ assigns an importance level to each candidate content span $c_i$. The conformal score $S_\beta$ of data point $(x, y^*)$ is set as the lowest relevance level where a fraction $\beta$ of ground-truth important content spans are retained, rounded up to the nearest whole number. If we sort the scores $R$ for each $c\in y^*$ so that $R_{(1)}\leq ... \leq R_{(|y^*|)}$, then
    \vspace{-2pt}
\begin{equation}\label{eq:s-beta}
  S_\beta(x, y^*) = R_{(r)}, \quad r = |y^*| - \lceil \beta\, |y^*| \rceil + 1.
\end{equation}
For the case $\beta=1$ of perfect recall, we write
\begin{equation}\label{eq:s}
S(x, y^*) = \min_{c \in y^{*}} R(c; x).
\end{equation}
 
\noindent A threshold $\hat{q}$ is calibrated on the conformal scores of a calibration dataset. For a new document $x_\text{test}$, any content span $c_i$ with relevance $R(c_i; x_\text{test}) < \hat{q}$ is filtered out. The remaining content spans form the prediction set $y_\text{test}\subseteq x$. This method provides a coverage guarantee that at least a fraction $\beta$ of important content per datapoint is retained with high probability \citep{kuwahara2025conformal},
    \vspace{-2pt}
\begin{equation}\label{eq:conformal_importance}
    \mathbb{P}\bigg[\frac{|y_\text{test}\cap y_\text{test}^*|}{|y_\text{test}^*|} \geq \beta\bigg] \geq 1 - \alpha.
        \vspace{-2pt}
\end{equation}
Coverage holds for \emph{any} relevance function $R$ under exchangeability of the documents $x$, but the quality of $R$---its power in classifying important content spans---determines \emph{conciseness} for the retained set.

Similar recall-style guarantees have been developed for conformal agent error attribution \citep{feng2026conformal}. In this task the relevance function predicts how likely a step in the agent's trace is to be a decisive error, and coverage ensures that decisive errors are contained in the prediction set.

Conformal factuality, conformal importance, and conformal agent error attribution all assume a fixed $S$ deriving from $R$, which is typically an LLM with a manually crafted prompt designed specifically for one definition of relevance. Generative tasks in this class have been studied generally by \citet{loaiza2026conf}. In this work, we construct a scoring function that generalizes across content selection tasks without per-domain engineering.

\textbf{Conformal ensembles.} When multiple scoring functions are available, they can be ensembled to produce stronger coverage guarantees, or more concise sets. Conformal score aggregation~\citep{ochoa2025conformal} establishes that score-level aggregation is strictly more efficient than set-level, yielding tighter prediction regions in classification and regression. \citet{gasparin2024merging} prove a $1{-}2\alpha$ coverage floor for majority-vote ensembling of prediction sets. Other score or $e$-value combination methods~\citep{luo2025conformal, alami2026symmetric} address classification and regression, not language tasks. In the language space, and building on conformal factuality, \citet{cherian2024large} propose a boosted linear combination of four heterogeneous scoring functions with level-adaptive conformal risk control~\cite{angelopoulos2024conformal}. Our work characterizes when ensembling provides a conciseness advantage in recall-oriented conformal language tasks. We discuss additional related work on conformal aggregation in Appendix \ref{sec:app-related}.

\textbf{ICL example selection.} A substantial literature studies which in-context examples to present to an LLM, spanning similarity-based retrieval~\citep{liu2022makes} and learned retrievers~\citep{rubin2022learning}. Ordering and input--label mapping substantially affect downstream performance~\citep{lu2022fantastically, min2022rethinking}, and diverse, balanced demonstrations improve compositional generalization~\citep{levy2023diverse}. Determinantal point process (DPP) based selection over in-context examples~\citep{ye2023compositional,sui2024lfr} is closest in spirit to our work, but instead of generation we consider relevance scoring within conformal frameworks.

\section{Theoretical Framework}
\label{sec:theory}

\begin{table*}[t]
\centering
\caption{Notation}\label{tab:notation}
\vspace{-6pt}
\footnotesize
\begin{tabular}{@{}ll@{\quad}ll@{}}
\toprule
$x$, $c$, $y^*$ & document, content span, relevant set & $S_\beta$ & conformal score (Eq.~\ref{eq:s-beta} for $\beta=1$) \\
$R_j$, $\bar R^{(K)}$ & relevance fn, mean ensemble & $\hat{q}$ & calibrated threshold \\
$K$, $k$ & \# relevance fns, ICL examples per fn & $\Pi$, $\mathcal{C}$ & ICL example pool, calibration set \\
$\alpha$, $\beta$ & miscoverage, recall target & $n$ &  calibration set size \\
$S$, $S^{(K)}$ & individual / ensemble floor at $\beta=1$ & $\mathrm{Comp}$ & complementarity (Def.~\ref{def:complementarity}) \\
$S_\text{max}, S_\text{min}$ & max/min individual floor at $K{=}2$ & $\delta^K(c)$ & floor margin (Prop.~\ref{prop:diminishing}) \\
\bottomrule
\end{tabular}
\vspace{-6pt}
\end{table*}

Considering recall-style coverage guarantees for content selection tasks (Eq.~\ref{eq:conformal_importance}), we replace a single relevance function $R$ with a $K$-strategy ICL ensemble. The remainder of this section asks three questions in turn: \emph{does ensembling preserve the $1-\alpha$ coverage guarantee?}; \emph{under which conditions does ensembling with $K=2$ improve conciseness?}; and \emph{when does adding an additional $R_{K+1}$ raise the ensembled conformal score?}

In this section we consider $\beta = 1$ for clarity before mentioning results for general $\beta$. See Appendix~\ref{sec:app-proofs} for extended details.

\subsection{Problem Setup}\label{ssec:setup}

 From Eq.~\ref{eq:s}, the conformal score $S(x, y^*)$ for a given relevance function $R(c;x)$ is defined as the \emph{floor}---the lowest relevance of any positive $c\in y^*$---which we write $S(R)$.  Raising the floor will in general also raise the calibrated threshold $\hat{q}$, which shrinks prediction sets, assuming relevance scores for $c\not\in y^*$ are unchanged. As long as coverage is unaffected, raising the floor under these conditions is desirable. We aim to show that ensembling different relevance scores can raise the floor while maintaining coverage. For $K$ distinct relevance functions $R_1, \ldots, R_K$, the mean ensemble $\bar R^{(K)}(c;x) \equiv \tfrac{1}{K}\sum_{j=1}^{K} R_j(c;x)$ has ensemble floor $S^{(K)} \equiv \min_{c \in y^*} \bar R^{(K)}(c;x)$.

\subsection{Validity via ensembling}\label{ssec:fusion}

We ask: \emph{does ensembling preserve the $1-\alpha$ coverage guarantee?} Set-level ensembling degrades coverage to $1{-}2\alpha$ \citep{gasparin2024merging}. Instead, we adopt score-level ensembling by taking the mean of several $R_j$, ensuring that they output relevance on a $[0,1]$ scale. Since $\bar R^{(K)}$ gives a fixed conformal score function via \Cref{eq:s}, standard split-conformal theory yields $1-\alpha$ coverage without degradation as long as exchangeability holds. We assume access to a pool of labeled examples $\Pi$ from which ICL examples are drawn for the $R_j$, disjoint from calibration and test sets. For per-sample ICL retrieval, conditioning on $\Pi = \pi$ restores exchangeability, and the law of total expectation lifts the conditional guarantee to an unconditional one (App.~\ref{sec:app-exchangeability}). Hence, our mean ensembling will preserve $1-\alpha$ coverage.

\subsection{Complementarity at $K=2$}\label{ssec:complementarity}

\emph{We ask: under which conditions does ensembling with $K=2$ improve conciseness?} Whether the ensemble floor $S^{(2)}$ improves over the individual floors depends on their values, and on how aligned the scorers' failures are; we capture these with $S_\text{max}, S_\text{min}$ and a new quantity $\mathrm{Comp}$, for \emph{complementarity}, then state the decomposition that combines them.

Concretely, let $S_\text{max} \equiv \max(S(R_1),\, S(R_2))$ and $S_\text{min} \equiv \min(S(R_1),\, S(R_2))$ denote the better and worse individual floors; their difference $S_\text{max} - S_\text{min} \geq 0$ is the \emph{floor gap}.

\begin{definition}[Complementarity]\label{def:complementarity}
For two relevance functions $R_1, R_2$ on input $x$,
\begin{multline}\label{eq:comp-def}
\mathrm{Comp}(R_1, R_2) \;\equiv\; \min_{c \in y^*}\bigl[R_1(c) + R_2(c)\bigr] \\
  - \bigl[S(R_1) + S(R_2)\bigr] \;\geq\; 0.
\end{multline}
\end{definition}

\paragraph{Interpretation.} $\mathrm{Comp}$ measures the extent to which the two relevance functions disagree on the lowest scoring positives $c\in y^*$.  $\mathrm{Comp} = 0$ when $R_1$ and $R_2$ share a lowest scoring positive; $\mathrm{Comp} > 0$ (complementarity) when their minimum is achieved on different positives, except possibly when $R_j$ has degeneracy across the $c_i$.

\begin{restatable}{lemma}{decomposition}
\label{lem:complementarity}
\text{(Ensemble floor advantage)} 
At $K = 2$ and $\beta = 1$,
\begin{equation}\label{eq:floor-decomp}
S^{(2)} \;=\; \tfrac{1}{2}\bigl(S_{\max} + S_{\min} + \mathrm{Comp}\bigr),
\end{equation}
and $S^{(2)} > S_{\max}$ iff $\mathrm{Comp} > S_{\max} - S_{\min}$.
\end{restatable}
\noindent See proof in App.~\ref{sec:app-complementarity}.
\paragraph{Interpretation.} The lemma tells us when a mean ensemble of two scorers has an advantage over both individual components by raising the floor: $\mathrm{Comp}$ must exceed the floor gap. Designing good ensembles thus means searching for scorers whose lowest scoring positives do not align. See App.~\ref{sec:app-comp-worked-example} for a content span-level example showing that the complementarity condition $\mathrm{Comp} > S_{\max} - S_{\min}$ can be achieved in practice.

\subsection{Diminishing returns}\label{ssec:diminishing}

\emph{We ask: when does adding an additional $R_{K+1}$ raise the ensemble floor $S^{(K)}$?} We study $\delta^K(c) \equiv \bar R^{(K)}(c;x) - S^{(K)}$, the margin between $c$'s average relevance score and the ensemble floor, which is non-negative for every $c\in y^*$.
\begin{restatable}{proposition}{diminishing}
\label{prop:diminishing}
\text{($O(\tfrac{1}{K})$ diminishing returns)} \\
We have $S^{(K+1)} > S^{(K)}$ iff $R_{K+1}(c) > S^{(K)} - K\,\delta^K(c)$ for every $c \in y^*$.  When the condition holds,
 \begin{equation}\label{eq:diminishing-bound}
0 < S^{(K+1)} {-} S^{(K)} \leq \frac{1 - S^{(K)}}{K + 1} = O\Big(\frac{1}{K}\Big).
\end{equation}
\end{restatable}
\noindent See proof in App.~\ref{sec:app-adding}. 
\paragraph{Interpretation.} Adding additional scoring functions to the ensemble can help when the new function $R_{K+1}$ assigns higher relevance to the current lowest scoring spans, i.e. those with $\delta^K(c)=0$. This again emphasizes the need for diversity among the relevance functions. Even when the condition for improvement holds, each additional scorer contributes proportionally less, and the marginal gain at step $K+1$ is at most $(1 - S^{(K)})/(K + 1)$. For small $K$, this can still be a meaningful improvement, but large $K$ is not necessary in practice.

\subsection{Extension to $\beta < 1$}

The theory developed above works in the setting of $\beta=1$, perfect recall of relevant content spans $c\in y^*$. Requiring perfect recall may be overly conservative, and result in larger than desirable prediction sets for a given coverage (\Cref{eq:conformal_importance}). Here we briefly summarize the extension of our theory to $\beta<1$, with full results in Appendix \ref{sec:app-comp-beta}.

To mirror the definition in \Cref{eq:s-beta}, we replace the minimum in $\mathrm{Comp}(R_1, R_2)$ by the $r$th order statistic (\Cref{eq:comp-beta}), which maintains the interpretation of measuring complementarity---how much two scoring functions disagree on the $r$th lowest scoring positive span. Then the score $S_\beta^{(2)}$, the $r$th lowest score from the mean ensemble $\bar R$, improves over $S_{\beta,\max}$ exactly when $\mathrm{Comp}_\beta >S_{\beta,\max} -S_{\beta,\min}$ (\Cref{prop:comp-beta}). In other words, the floor is raised when at least $\lceil \beta |y^* |\rceil$ of the positive spans satisfy $\bar R(c) > S_{\beta, \max}$.

For general $K$, adding an additional scoring function to an existing ensemble will give $S_\beta^{(K+1)}> S_\beta^{(K)}$ when at least $\lceil \beta |y^* |\rceil$ of the positive spans satisfy $R_{K+1}(c) > S^{(K)}_{\beta}-K\delta_\beta^{K}(c)$, where now the margin is $\delta_\beta^K(c) = \bar{R}^{(K)}(c;x) - S_\beta^{(K)}$. In this case, the improvement is bounded above by $(1-S_\beta^{(K)})/(K+1)$ which again shows $O(\tfrac{1}{K})$ diminishing returns (\Cref{prop:adding-beta}).

\vspace{-2pt}
\section{Method - Conformal Relevance}
\label{sec:method}
\vspace{-4pt}

Guided by the theoretical framework of \Cref{sec:theory}, we compose an ensemble of mechanistically diverse ICL strategies into a mean scoring function $\bar R$, with the end-to-end procedure summarized in Algorithm~\ref{alg:pipeline}.  This section defines the data split (\Cref{subsec:split}), the ICL-based relevance functions (\Cref{subsec:scoring}), and the ICL-selection strategies (\Cref{subsec:composition}). The design and hyperparameters for our reference implementation are held fixed across all main results for every dataset and task, with key ablations in \Cref{sec:ablations}.

\begin{algorithm}[t]
\caption{Conformal calibration and prediction for content selection with a $K$-scorer ensemble.}
\label{alg:pipeline}
\small
\begin{algorithmic}[1]
\State \textbf{Inputs:} ICL pool $\Pi$, calibration set $\mathcal{C}$, $K$ ICL selection strategies $\sigma_j$, ICL count $k$, miscoverage $\alpha$, recall target $\beta$
\Statex \textit{Offline calibration on $\mathcal{C}$}
\For{each $x \in \mathcal{C}$}
  \ForAll{$j \in \{1, \ldots, K\}$}
    \State $\mathrm{ICL}_j \leftarrow \sigma_j(\Pi, x)$ \Comment{draw $k$ examples}
    \State $R_j(\cdot; x) \leftarrow \mathrm{LLM}(\cdot; x, \mathrm{ICL}_j)$ \Comment{define scorers}
  \EndFor
  \State $\bar R(\cdot; x) \leftarrow \tfrac{1}{K}\sum_{j=1}^{K} R_j(\cdot; x)$ \Comment{mean ensemble}
  \State compute $S_\beta(x)$ using $\bar R(\cdot; x)$ \hfill\textit{(\Cref{eq:s-beta})}
\EndFor
\State $\hat{q} \leftarrow \lfloor \alpha(n{+}1) \rfloor$-th order statistic of
\Statex\quad $\{S_\beta(x) : x \in \mathcal{C}\}$ \hfill\textit{(recall-style conformal quantile)}
\Statex \textit{Online prediction on test document $x$}
\State Repeat steps~3--7 on a test document $x$ to obtain $\bar R(\cdot; x)$
\State \textbf{Output:} $\{c \in x : \bar R(c; x) \geq \hat{q}\}$
\end{algorithmic}
    \vspace{-4pt}
\end{algorithm}

\subsection{Data Split}
\label{subsec:split}
The datasets we experiment on are shown in \Cref{tab:datasets}. Each dataset is partitioned into three subsets: a \textbf{pool} $\Pi$ supplies ICL examples, a \textbf{calibration set} $\mathcal{C}$ is used to set the conformal threshold $\hat{q}$, and a held-out \textbf{test set} is used for all reported results. We use $|\Pi|{=}50$ for single-intent datasets, or $20$ examples per intent for multi-intent ones, $|\mathcal{C}|=n=100$, and assign the remainder to test. Calibration and test are drawn uniformly at random from the available data so that the exchangeability assumption holds; $\Pi$ is stratified by intent for the multi-intent datasets (PUMA, SubSumE, ContractNLI, Evidence Inference), so each of the $k$ ICL examples per strategy are drawn from the same-intent subpool of $\Pi$ as the test sample. The labeling budget ranges from $150$ to $440$ at maximum (ContractNLI with $17$ intents).

\subsection{ICL-Based Relevance Scoring}
\label{subsec:scoring}

A relevance scoring function $R(c;x)\in[0,1]$ assigns a score to every defined content span $c$ of document $x$ (e.g. each sentence). We instantiate $R$ using LLMs, but not with a manually crafted prompt comprising the definition of what should be considered relevant for a given task (e.g. describing what constitutes important information to summarize), but only with $k$ ICL examples sampled from $\Pi$, each a labeled document with its content spans and binary relevance labels. Avoiding manual prompt engineering means the same scoring function can be applied across content selection tasks, from extractive QA to summarization to PII detection. Details on the construction of the ICL prompts, such as prompt format, windowed compression, optional task hint, and the output schema are elaborated in detail in Appendix \ref{sec:prompting_specification}.

\begin{table*}[t]
\vspace{-2pt}
\caption{Dataset statistics. Full processing details in Appendix~\ref{sec:app-reproducibility}.}
\label{tab:datasets}   
\vspace{-6pt}
\centering
\setlength{\tabcolsep}{2.5pt}
\resizebox{\textwidth}{!}{
\begin{tabular}{lllrrrrrr}
\toprule
Dataset & Domain & Task & Avg \# Spans & Pos. Rate & Intents & \# ICL Pool & \# Cal & \# Test \\
\midrule
ECTSum \citep{mukherjee-etal-2022-ectsum}        & Financial     & Summarization          & 46  & 9.8\%  & ---    & 50  & 100 & 2{,}275  \\
SubSumE \citep{yadav2021subsume}                  & Encyclopedic  & Query-focused summ.    & 461 & 2.5\%  & 10     & 200 & 100 & 1{,}900  \\
PUMA \citep{naik2024no}                           & Medical       & Perspective-based QA   & 14  & 36.0\% & 5      & 100 & 100 & 6{,}082  \\
PhysioNet \citep{PhysioNet-deidentifiedmedicaltext-1.0}      & Medical       & PII detection & 23  & 13.9\% & ---    & 50  & 100 & 637      \\
HotpotQA \citep{yang2018hotpotqa}                 & General       & Multi-hop QA           & 41  & 6.3\%  & ---    & 50  & 100 & 10{,}000 \\
Evidence Inf.\ \citep{deyoung2020evidence}        & Medical       & Evidence extraction    & 155 & 2.0\%  & 8      & 160 & 100 & 9{,}036  \\
ContractNLI \citep{koreeda2021contractnli}        & Legal         & NDA clause scoring     & 84  & 2.7\%  & 17     & 340 & 100 & 5{,}733  \\
\bottomrule
\end{tabular}
}
\vspace{-10pt}
\end{table*}

\subsection{ICL Example Selection Strategies}
\label{subsec:composition}

Building on the potential benefits of scoring ensembles described in \Cref{sec:theory}, our aim is to design diverse relevance scoring functions $R_j$ that can be combined via mean ensembling as $\bar R^{(K)}$, which then goes into the conformal score $S_\beta^{(K)}$ for content selection (\Cref{eq:s-beta}). To realize the benefits in Lemma~\ref{lem:complementarity}, we need relevance functions with mechanistically uncorrelated errors (low scores assigned to relevant spans $c\in y^*$). However, there are diminishing returns when adding many scorers to the ensemble (\Cref{prop:diminishing}), and added costs. We therefore strike a balance by constructing $K=4$ ICL-selection strategies $\sigma_j$ spanning distinct retrieval signals. Each strategy determines how ICL examples are selected from the pool given an input $x$, which then go into a fixed prompt format. Hence, each strategy $\sigma_j$ gives rise to a distinct scoring function $R_j$.
\begin{itemize}[noitemsep,topsep=0pt,leftmargin=*]
    \item \texttt{anchor\_dpp} embeds the query document $x$, finds the document from the ICL pool with highest cosine similarity, and sets that as the anchor. Then $k-1$ other documents are selected from the pool via a DPP conditioned on the anchor.
    \item \texttt{pattern\_dpp} embeds each content span in a pool document, and computes centroid embeddings of the positive ($c\in y^*$) and negative ($c\not\in y^*$) spans. The difference of positive and negative centroids forms a relevance direction vector. We then select $k$ examples via a DPP over the relevance directions.
    \item \texttt{bm25} is lexical top-$k$ retrieval with BM25 \citep{robertson2009bm25}, using $x$ as the query into the index over ICL pool documents;
    \item \texttt{random} selects $k$ documents from $\Pi$ uniformly at random, and serves as an ensemble regularizer.
\end{itemize}
Extended descriptions of the ICL selection strategies are in Appendix~\ref{sec:app-strategies}.

Because every $R_j$ outputs on the same $[0,1]$ scale, the natural ensembling method is the element-wise mean $\bar R(c;x)=\tfrac{1}{K}\sum_j R_j(c;x)$. Our proposed conformal scoring function for arbitrary content selection tasks is \textbf{Ens4}, the mean ensemble of the four ICL-selection strategies mentioned above, using $k{=}2$ ICL examples per scorer (Algorithm~\ref{alg:pipeline}). Validity of the coverage guarantee follows from our argument in \Cref{ssec:fusion}.

\begin{table*}[!t]
\caption{Main MAP results on the seven datasets. \textbf{Bold} marks the best performing setting. $\Delta$ columns report the improvement of Ens4 over baselines with 95\% bootstrap CI in brackets.}
\vspace{-6pt}
\label{tab:main}
\centering
\setlength{\tabcolsep}{2pt}
\resizebox{\textwidth}{!}{
\begin{tabular}{lcclcll}
\toprule
Dataset & ICL0 & Avg Single ($k{=}2$) & Best Single (Config) & Ens4 & $\Delta$ vs ICL0 [95\% CI] & $\Delta$ vs Best Single \\
\midrule
ECTSum
  & 0.350 & 0.392
  & 0.469~(\texttt{anchor\_dpp}, $k{=}3$)
  & \textbf{0.516} & $+0.166_{[+0.150,+0.181]}$ ($+47\%$) & $+0.047_{[+0.032,+0.063]}$ ($+10\%$) \\
Evidence Inf.
  & 0.206 & 0.201
  & 0.213~(\texttt{anchor\_dpp}, $k{=}1$)
  & \textbf{0.304} & $+0.099_{[+0.093,+0.105]}$ ($+48\%$) & $+0.091_{[+0.082,+0.100]}$ ($+43\%$) \\
HotpotQA
  & 0.749 & 0.740
  & 0.749~(\texttt{anchor\_dpp}, $k{=}2$)
  & \textbf{0.839} & $+0.090_{[+0.084,+0.096]}$ ($+12\%$) & $+0.090_{[+0.084,+0.097]}$ ($+12\%$) \\
PhysioNet
  & 0.764 & 0.766
  & 0.795~(\texttt{random}, $k{=}3$)
  & \textbf{0.879} & $+0.115_{[+0.083,+0.147]}$ ($+15\%$) & $+0.080_{[+0.049,+0.112]}$ ($+10\%$) \\
PUMA
  & 0.648 & 0.752
  & 0.777~(\texttt{anchor\_dpp}, $k{=}2$)
  & \textbf{0.814} & $+0.165_{[+0.156,+0.175]}$ ($+25\%$) & $+0.036_{[+0.027,+0.045]}$ ($+5\%$) \\
SubSumE
  & 0.289 & 0.354
  & 0.373~(\texttt{bm25}, $k{=}3$)
  & \textbf{0.464} & $+0.170_{[+0.153,+0.187]}$ ($+59\%$) & $+0.089_{[+0.072,+0.107]}$ ($+24\%$) \\
ContractNLI
  & 0.718 & 0.711
  & 0.733~(\texttt{pattern\_dpp}, $k{=}3$)
  & \textbf{0.828} & $+0.109_{[+0.099,+0.120]}$ ($+15\%$) & $+0.095_{[+0.085,+0.105]}$ ($+13\%$) \\
\bottomrule
\end{tabular}
}
\vspace{-10pt}
\end{table*}

\begin{table*}[t]
\vspace{-10pt}
\caption{Metric reported is MAP. Left: Ablation of ICL examples per strategy ($k$) in Ens4. Right: Ablation of ensemble size ($K$) in Ens$K$. For $K=2$ and $K=3$ we show the single best combination of all 4 ICL strategies over every combination.}
\label{tab:ablation-k-K}
\vspace{-6pt}
\renewcommand{\arraystretch}{0.95}
\centering
\resizebox{0.85\textwidth}{!}{
\begin{tabular}{lcccccccc}
\toprule
& \multicolumn{5}{c}{\emph{$k$ ablation (K=4 fixed)}}
& \multicolumn{3}{c}{\emph{Oracle-best ensemble size ($k{=}2$ fixed)}} \\
\cmidrule(lr){2-6} \cmidrule(lr){7-9}
Dataset & $k{=}1$ & $k{=}2$ & $k{=}3$ & $k{=}5$ & $k{=}8$
        & \multicolumn{1}{c}{Best Ens2} & \multicolumn{1}{c}{Best Ens3} & \multicolumn{1}{c}{Ens4} \\
\midrule
ECTSum          & 0.466 & 0.516          & 0.507 & \textbf{0.518} & 0.517
                & 0.492 & 0.509 & \textbf{0.516}  \\
Evidence Inf.   & \textbf{0.307} & 0.304 & 0.303 & 0.300 & 0.293
                & 0.261 & 0.288 & \textbf{0.304}  \\
HotpotQA        & 0.832 & \textbf{0.839} & 0.836 & \textbf{0.839} & 0.835
                & 0.812 & 0.832 & \textbf{0.839} \\
PhysioNet       & 0.864 & \textbf{0.879} & 0.874 & 0.869 & 0.861
                & 0.856 & 0.874 & \textbf{0.879} \\
PUMA            & 0.812 & \textbf{0.814} & 0.813 & 0.810 & 0.806
                & 0.800 & 0.808 & \textbf{0.814} \\
SubSumE         & 0.439 & \textbf{0.464} & 0.458 & 0.460 & 0.460
                & 0.427 & 0.452 & \textbf{0.464} \\
ContractNLI     & 0.821 & 0.828          & \textbf{0.833} & 0.832 & 0.820
                & 0.787 & 0.816 & \textbf{0.828} \\
\bottomrule
\end{tabular}
}
\renewcommand{\arraystretch}{1.0}
\vspace{-6pt}
\end{table*}

\vspace{-4pt}
\section{Experiments}
\vspace{-3pt}
\label{sec:experiments}

\vspace{-2pt}
\subsection{Experimental Setup}
\label{sec:setup}
\vspace{-2pt}
Code is available at \href{https://github.com/layer6ai-labs/conformal-relevance}{github.com/layer6ai-labs/conformal-relevance}. We evaluate on seven sentence-level relevance datasets spanning five domains (financial, encyclopedic, medical, general QA, and legal), four task types (summarization, question answering, entity/span detection, and clause scoring), and document lengths from 14 to 461 sentences (Table~\ref{tab:datasets}). Four datasets (HotpotQA, ECTSum, ContractNLI, Evidence Inference) were reformulated to sentence-level binary relevance with per-dataset construction details in App.~\ref{sec:app-datasets}.

\paragraph{Metrics.}
We report two scoring-function quality measures. \emph{Mean Average Precision} (MAP), the mean of per-sample AP across the test set, ranks the scoring function across the full score distribution and is the primary metric in Sections \ref{sec:main-results} and \ref{sec:ablations}. \emph{Conciseness}, $\mathbb{E}_{x \sim \text{test},\ \text{MC splits}}\!\left[1 - \lvert C_{\hat q}(x)\rvert/\lvert x\rvert\right]$, is the expected fraction of sentences removed by the calibrated prediction set $C_{\hat q}(x)$ at fixed $(\alpha, \beta)$, with the expectation taken over the test set. To account for variance in the conformal algorithm due to dataset splitting, we repeat the experiments with 400 random calibration/test splits and average metrics over these runs. Higher is better for both MAP and conciseness. We further examine the $S_\beta$ order statistic underlying the calibrated threshold $\hat q$ on a per-sample level in \Cref{sec:ablations}. Empirical coverage at target level $1{-}\alpha$ is reported in \Cref{sec:conformal-results} as a validity check on the theoretical guarantees which must hold for any fixed scoring function under exchangeability.

\paragraph{Baselines.}
\textbf{ICL0}: a manually crafted per-dataset prompt that thoroughly describes what should be considered relevant for that task, with no ICL examples (full prompts in App.~\ref{sec:app-prompts}). \textbf{Best Single}: a post-hoc oracle, the single highest MAP strategy configuration of the four ICL selection strategies, selected per-dataset with tuned hyperparameter $k$. The \emph{Best Single (Config)} column in Table~\ref{tab:main} reports the winning ($\sigma_j$, $k$) pair, and \emph{Avg Single} reports the mean of the four strategies MAPs at $k{=}2$.

\paragraph{Default configuration.}
Unless stated otherwise, all \textbf{Ens4} experiments use $K{=}4$ scoring functions in the mean ensemble, created from the four ICL selection strategies (\texttt{anchor\_dpp}, \texttt{pattern\_dpp}, \texttt{bm25}, \texttt{random}), $k{=}2$ ICL examples per strategy, Gemini-2.5-Flash-Lite \citep{google2025gemini25flash} as the scoring LLM, with fixed seeds. Single-strategy baselines use the same model and seeds. We sweep over the hyperparameter $k$ for values $\{1, 2, 3, 5, 8\}$ with each strategy, both to find what value drives the Best Single oracle, and for the $k$-ablation of \Cref{sec:ablations}.

\subsection{Main Results}
\label{sec:main-results}
\vspace{-4pt}
Table~\ref{tab:main} displays our main results for MAP across all seven datasets, comparing Ens4 to the two baseline strategies and average performance. We emphasize that a single Ens4 configuration is used for all tests---only the datasets change. Ens4 improves MAP on all seven datasets, beating both ICL0 and the per-dataset best single strategy. To make the comparisons more explicit, we show the amount of improvement ($\Delta$) of Ens4 over ICL0 and Best Single, with 95\% confidence intervals computed over 10,000 bootstrap samples of the 400 measurements comprising each average value in the table. Improvements over ICL0 range from $+0.090$ to $+0.170$ MAP ($+12\%$ to $+59\%$), with all confidence intervals excluding zero. This shows that ICL examples can capture the meaning of relevance across a variety of tasks better than a manually crafted description. Additional visualizations of the data are shown in App.~\ref{sec:app-strategy-matrix}.

Several other patterns stand out from the data.
\textbf{Complementarity.}\ For HotpotQA no single ICL strategy beats ICL0, yet Ens4 improves $+0.090$ MAP. Additionally, even though the $k=2$ single strategies are weaker on average (Avg Single) than the Best Single configuration with $k$ tuned, Ens4 using the $k=2$ version of the strategies still outperforms. These observations demonstrate that ensembling over strategies of similar strength can lead to improvements due to complementarity.\\
\textbf{Winner heterogeneity.}\ No single sub-strategy dominates across datasets (Best Single column). Each of the four selection strategies is the best for at least one dataset, including the \texttt{random} strategy. This demonstrates that it can be difficult to choose a single ICL selection strategy that works broadly, whereas our Ens4 ensemble is consistently strong.\\
\textbf{Statistical significance and robustness.}\ The bootstrap CIs in Table~\ref{tab:main} lower-bound the improvement over ICL0 strictly above zero on every dataset. Because the seven benchmarks share a scoring LLM, prompt template, and seeds, we treat the 7/7 directional pattern as a consistency check rather than as seven independent trials.

Ens4 comes with higher computational burden than ICL0, using four calls to an LLM scorer instead of one, with negligible ICL-retrieval overhead. We analyze structural and empirical computational costs in App.~\ref{sec:app-strategy-matrix}.

\subsection{Ablations and Validation}\label{sec:ablations}
We ablate key hyperparameters ($k$, $K$), validate the theory, isolate retrieval-induced diversity, and report framework robustness.
\vspace{-4pt}
\paragraph{ICL examples per strategy ($k$).}
Existing research on general NLP tasks has reported inconsistent findings on whether additional ICL examples are helpful or harmful \citep{chen2023many, zou2025many}. The optimal number of ICL examples appears to be task-dependent, so we test here for conformal content selection tasks. In \Cref{tab:ablation-k-K} (left) we vary $k \in \{1, 2, 3, 5, 8\}$ with $K{=}4$ fixed. $k{=}1$ underperforms $k{=}2$ on 6/7 datasets, likely due to insufficient guidance for pattern recognition. Higher values ($k \geq 5$) degrade performance on most datasets, which could be from DPP-kernel saturation and prompt-length effects starting to dilute instructions to the model. Instead, small values of $k \in \{2, 3\}$ show the most consistent and strong performance. We use $k{=}2$ in the default Ens4 since it is the more efficient of the two similar options.

\begin{figure*}[t]
\centering
\begin{subfigure}[t]{0.49\linewidth}
  \centering
  \includegraphics[width=\linewidth, height=4.8cm, keepaspectratio]{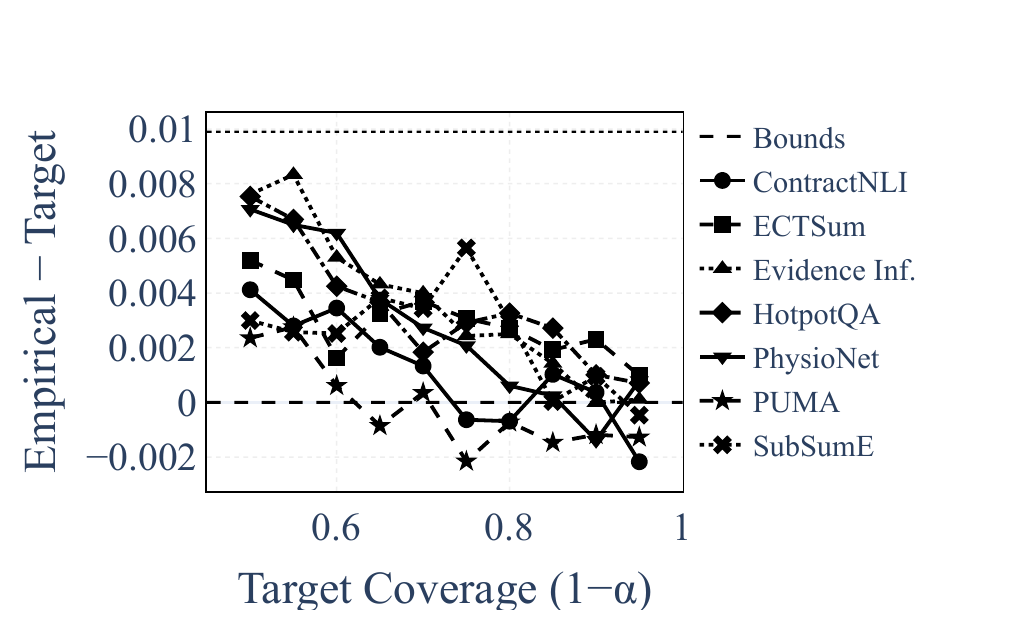}
  \caption{Empirical minus target coverage vs.\ $1{-}\alpha$. All datapoints remain within 1\,pp. of the target. Dashed: bounds from Theorem~1 of \citet{kuwahara2025conformal}.}
  \label{fig:coverage-diff}
\end{subfigure}\hfill
\begin{subfigure}[t]{0.49\linewidth}
  \centering
  \includegraphics[width=\linewidth, height=4.8cm, keepaspectratio]{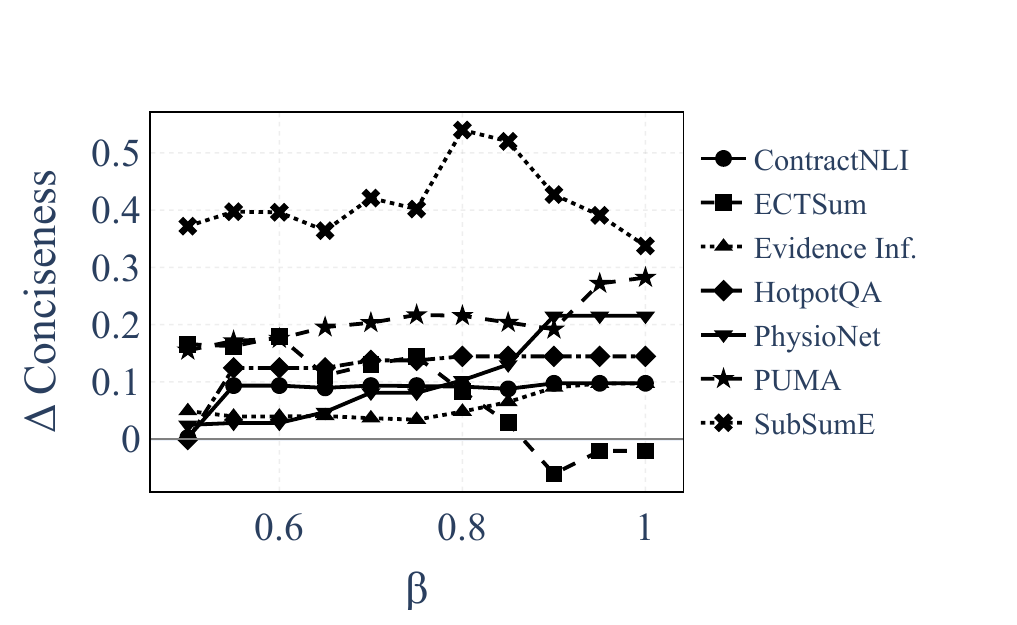}
  \caption{Conciseness gain $\Delta{=}\text{Ens4}{-}\text{ICL0}$ vs.\ $\beta$ at
    $\alpha{=}0.20$.}
  \label{fig:conciseness}
\end{subfigure}
\vspace{-6pt}
\caption{Conformal coverage validity (left) and Ens4--ICL0 conciseness gain
  (right) across all 7 datasets ($n{=}100$).}
\label{fig:coverage-combined}
\vspace{-10pt}
\end{figure*}

\vspace{-4pt}
\paragraph{Ensemble size ($K$) and composition.}
With $k=2$ fixed, we run Ens$K$ across all non-singleton combinations of the four ICL selection strategies (11 total configurations). Results are shown in \Cref{tab:ablation-k-K} (right). Our default Ens4 performs best on every dataset, showing that ensembling more scoring functions, even potentially weak ones like \texttt{random}, remains beneficial due to complementarity gains. However, marginal gains are strictly decreasing across all 7 datasets going from $K=3\to 4$ compared to $2\to 3$, consistent with \Cref{prop:diminishing}'s $O(\tfrac{1}{K})$ bound and diminishing returns. See  App.~\ref{sec:app-composition} for further breakdown.
\paragraph{Theory validation.}
App.~\ref{sec:app-theory-validation} demonstrates empirical checks on the conditions and assumptions made throughout \Cref{sec:theory}. Notably, we check how often the advantage condition from \Cref{lem:complementarity} is satisfied, that $\mathrm{Comp} > S_{\max} - S_{\min}$, and how often this co-occurs with a measured advantage (see \Cref{tab:app-prop1-agreement}). When the condition is met, advantage is observed over 99\% of the time.
\vspace{-2pt}
\paragraph{ICL complementarity.}
We test the importance of complementarity compared to other ICL factors: quantity, supervision, and increased compute.

\emph{Complementarity vs.\ ICL quantity.} Ens4 at $k{=}2$ deploys $K \cdot k = 8$ ICL examples for each query $x$, spread out over 4 scoring functions. This can be directly compared to Best Single at $k{=}8$, using the entire ICL budget in a single retrieval signal (App.~\ref{sec:app-quantity}, Table~\ref{tab:app-k-scaling-strategies}). Still, Ens4 outperforms by $+0.070$ to $+0.119$ MAP on every dataset.

\emph{Complementarity vs.\ labeled data.}
To test whether Ens4's gains can be explained by access to additional labeled examples, we train a supervised classifier using the available ICL pool and calibration data, which replaces the LLM-based ICL mechanism for scoring. Ens4 outperforms the classifier on 6/7 datasets (App.~\ref{sec:app-classifier}, \Cref{tab:app-classifier}).

\emph{Complementarity vs.\ stochasticity and compute.}
We replace retrieval diversity with stochastic diversity by ensembling scoring functions with different temperature values $T \in \{0.3, 0.5, 0.8, 1.0\}$ for a single ICL selection strategy. These approaches have matched LLM compute budgets, isolating the type of diversity in the scoring functions. Results in App.~\ref{sec:app-equal-compute}, Table~\ref{tab:app-equal-compute-ci} show that Ens4 outperforms the stochastic-diversity baseline in 12/16 settings, with overlapping CIs in the other 4. Individual temperature runs show relatively small MAP differences across $T$ (Table~\ref{tab:app-equal-compute-spread}). We additionally ensemble four temperature-varied ICL0 calls. This baseline directly matching Ens4's LLM call count, but without the ICL examples or retrieval complementarity. Ens4 significantly outperforms this baseline on 6/7 datasets, with a statistical tie on ContractNLI (Table~\ref{tab:app-icl0-equal-compute}).

Together these controls suggest retrieval-induced complementarity (Lemma~\ref{lem:complementarity}) as the operative mechanism for improved MAP of the ensembled conformal score function.

\paragraph{Robustness.} We vary other design choices of our method to verify robustness.

\emph{LLM:} We replace Gemini-2.5-Flash-Lite with Llama3-8B and Qwen3-8B on 4 representative datasets. Ens4 outperforms ICL0 on 7/8 (model, dataset) combinations and Best Single on 8/8. We further evaluate GPT-5.6-terra on seeded 500-sample test subsets of all seven datasets, where Ens4 outperforms both baselines on 7/7 datasets. Results are shown in App.~\ref{sec:app-cross-model}, Table~\ref{tab:app-cross-model}.

\emph{Task hint:} The default Ens4 configuration includes in the scoring model prompt a one-line specification of the task. We test whether ICL examples are sufficient to describe the task on their own by removing the hint altogether, making the configuration and prompts completely dataset-agnostic. Results are shown in App.~\ref{sec:app-task-hint}, \Cref{tab:hint-ablation}. 6/7 datasets stay within $0.04$ MAP when the hint is removed. Only PhysioNet, which tests non-semantic personal health information, is materially affected, collapsing from $0.879 \to 0.549$ MAP.

\subsection{Conformal Coverage and Conciseness}
\label{sec:conformal-results}
The preceding results focused on MAP, which is a measure of the accuracy of the conformal score function in isolation. We now evaluate the entire conformal pipeline which produces sets of content spans with coverage guarantees. First, we sweep over $\alpha \in \{0.05, \ldots, 0.50\}$, with fixed $\beta{=}0.8$ to calibrate $\hat q$ on the calibration set, and use it to generate prediction sets on the test data. This process is repeated 400 times with random splits of calibration and test data (while the ICL pool remains fixed), with results averaged. quality measure.

\paragraph{Coverage.} \Cref{fig:coverage-diff} shows the average empirical coverage of prediction sets produced with the Ens4 conformal score as $1-\alpha$ is varied. The dashed upper and lower bounds reflect the coverage guarantee of the algorithm we employ \citep{kuwahara2025conformal}. Empirical coverage remains within 1\,pp of the target $1{-}\alpha$ coverage for all datasets (Figure~\ref{fig:coverage-diff}). The coverage random variable has variance that depends on calibration set size, but even though we use a modest 100 labeled samples, we find empirical coverage is consistently close to the target value.

\paragraph{Conciseness.} \Cref{fig:conciseness} measures conciseness - the fraction of sentences filtered out of the prediction set. Given a fixed coverage level, smaller prediction sets are more useful, and reflect a more powerful and confident scoring function. Compared to ICL0, Ens4 removes more content at $1-\alpha =0.8$ for all datasets at almost all tested values of $\beta$. This shows that the MAP gains of Ens4 do translate into more useful prediction sets. We perform a larger sweep over $\alpha$ values in App.~\ref{sec:app-alpha-sweep}.

\paragraph{Alternative conformal ensemble rules.}
We additionally compare mean ensembling with alternative conformal ensemble rules, using the same four Ens4 constituent scoring functions in each case. Mean ensembling stays closest to the target coverage while achieving the highest conciseness. Majority voting under-covers at $\alpha{=}0.2$, whereas more conservative rules over-cover and produce less concise prediction sets (App.~\ref{sec:app-conformal-ensembles}, \Cref{tab:conformal-ensemble-rules}).

\section{Conclusion}\label{sec:conclusion}

Many NLP tasks share a common core objective --- extract relevant content from a document. The definition of relevance differs from one task to the next. In this work we introduced the \emph{Conformal Relevance} framework for providing coverage guarantees over extractive content selection tasks with minimal manual input, instead relying on in-context learning to define relevance implicitly for any given task. Compared to past work, we replace task-specific scoring functions (an LLM with a hand-crafted prompt) by a mean ensemble of four ICL example selection strategies, and find that a single configuration outperforms bespoke prompts on seven datasets spanning five domains. Consequently, for a range of coverage and recall levels, Conformal Relevance filters out substantially more irrelevant content. We presented theoretical arguments to capture when a mean ensemble should improve, and the limitations for how much improvement can be gained by adding addition scorers. Empirically, we studied the importance of complementarity amongst the ICL selection strategies, and ablated this against quantity and stochasticity controls.

We pose three extensions for future work: replace the fixed mean aggregator with an adaptive scheme similar to \citep{angelopoulos2022learn}, which would require an additional labeled data subset; extend binary relevance to graded labels (tiered or continuous) via conformal risk control \citep{angelopoulos2024conformal}; or to expand scope from purely extractive content selection to general generative tasks \cite{loaiza2026conf}.

\section*{Limitations}
\textbf{AI Use Disclosure:} We used coding and research agents to support the theoretical discussions and experimental verification of this work. We used LLM tools to assist with writing.

The Conformal Relevance framework assumes access to ${\sim}150$--$440$ labeled examples (the ICL pool plus 100 sample calibration set) which may need to be manually curated in practice; a static relevance criterion that does not drift between calibration and test to ensure exchangeability; and a one-line task hint which is manually crafted to explain what relevance means for a dataset, although this is optional and our method was largely unchanged on 6/7 datasets without it.

Our main results were reported on Gemini-2.5-Flash-Lite, with ablations on two 8B open-source LLMs. This is a limited set of models that have been surpassed in strength since this work was performed.

Our Ens4 ensemble of scoring functions requires $K=4$ LLM calls per document, compared to one for the simpler ICL0 baseline. This is an example of test-time compute, where spending additional computation with LLMs can improve performance, but of course increases costs.

All experiments fix temperature $T{=}0$, seed, model version, and prompt templates, however there is residual variation from the hosted LLM's server-side execution which is outside our control and can shift individual sample scores between re-runs. Replication should target aggregate MAP rather than per-sample scores.

The $1{-}\alpha$ guarantee for Conformal Relevance is marginal over the test distribution. Conditional coverage, for example over intents, document-length strata, or protected attributes, is not implied. Methods such as Mondrian conformal prediction~\citep{vovk2003mondrian} can be used when conditional coverage is required.

The ICL selection strategies select examples from the labeled pool, but on datasets where positives are very sparse, the effective number of positive ICL candidates may be small, which may prevent the LLM from seeing sufficiently informative context. Larger or more carefully constructed pools may narrow this gap, at proportionally higher labeling cost.

We do not foresee any particular societal risks from use of our methods.

\bibliography{bib}

\clearpage
\appendix
\section*{Appendix Contents}

\noindent\textit{Related Work (referenced by \Cref{sec:background})}
\begin{itemize}[leftmargin=2.2em,itemsep=1pt,topsep=2pt,parsep=0pt]
  \item[\ref{sec:app-related}] Additional Related Work
\end{itemize}
\noindent\textit{Theory (referenced by \Cref{sec:theory})}
\begin{itemize}[leftmargin=2.2em,itemsep=1pt,topsep=2pt,parsep=0pt]
  \item[\ref{sec:app-exchangeability}] Detailed Exchangeability Argument
  \item[\ref{sec:app-proofs}] Ensemble Theory — Full Proofs and Additional Results
\end{itemize}

\noindent\textit{Method (referenced by \Cref{sec:method})}
\begin{itemize}[leftmargin=2.2em,itemsep=1pt,topsep=2pt,parsep=0pt]
  \item[\ref{sec:app-reproducibility}] Reproducibility Details
  \item[\ref{sec:prompting_specification}] Prompting and ICL Specification
  \item[\ref{sec:app-strategies}] ICL Selection Strategy Details
  \item[\ref{sec:app-prompts}] ICL0 Prompts
\end{itemize}

\noindent\textit{Experiments (referenced by \Cref{sec:experiments})}
\begin{itemize}[leftmargin=2.2em,itemsep=1pt,topsep=2pt,parsep=0pt]
  \item[\ref{sec:app-strategy-matrix}] Main Empirical Results
  \item[\ref{sec:app-robustness}] Robustness Evaluations
  \item[\ref{sec:app-theory-validation}] Empirical Validation of Ensemble Theory\item[\ref{sec:app-alpha-sweep}] Conformal Conciseness Across $\alpha$
\end{itemize}

\ \\ 
\ \\

\section*{Notation Used in the Appendix}
\label{sec:app-notation}

\footnotesize
\begin{tabular}{@{}lll@{}}
\toprule
Symbol & Meaning & First defined \\
\midrule
\multicolumn{3}{l}{\emph{Span and document level}} \\
$x$, $c$              & Document, content span & §\ref{sec:background} \\
$y^*$                 & Relevant (positive) sentence set & §\ref{sec:background} \\
$R_j(c; x)$           & $j$-th relevance function, $\in [0,1]$ & §\ref{sec:background} \\
\midrule
\multicolumn{3}{l}{\emph{Conformal calibration}} \\
$\alpha$, $\beta$     & Coverage miscoverage budget; recall target & §\ref{sec:background} \\
$r$                   & Order-statistic index: $r = |y^*| - \lceil \beta |y^*| \rceil + 1$ & §\ref{sec:background} \\
$\hat{q}$             & Calibrated conformal threshold & §\ref{sec:background} \\
$\Pi$                 & Pool of ICL examples & \S\ref{sec:theory} \\
\midrule
\multicolumn{3}{l}{\emph{Ensemble level}} \\
$K$, $k$              & Number of scoring functions $R_j$; ICL examples per strategy & §\ref{sec:background} \\
$\bar{R}^{(K)}$ & Mean ensemble of scoring functions; $\bar{R}^{(K)} = \tfrac{1}{K}\sum_j R_j$ & §\ref{sec:theory} \\
$g$                   & Score aggregation function (we fix $g = \mathrm{mean}$) & §\ref{sec:theory} \\
$S_\beta$, $S_\beta^{(K)}$       & Conformal score using scoring function $R_j$, $\bar R^{(K)}$ & §\ref{sec:theory} \\
$\delta^K(c)$, $\delta_\beta^K(c)$         & Per-sentence margin above floor: $\bar{R}^{(K)}(c) - S^{(K)}$, etc. & §\ref{sec:theory} \\
$S_{\max}$, $S_{\min}$ & $\max$, $\min$ of $S$ across the $K{=}2$ pair & §\ref{sec:theory} \\
$\mathrm{Comp}$       & Complementarity measure & §\ref{sec:theory} \\
$A_j$                 & Set of minimizers for $R_j$: $\{c \in y^* : R_j(c) = S(R_j)\}$ & App~\ref{sec:app-complementarity} \\
\bottomrule
\end{tabular}
\normalsize

\clearpage
\section{Additional Related Work}
\label{sec:app-related}

There are two major approaches to ensembling applied to conformal prediction: set-level and score-level. In set-level prediction, multiple scoring functions each run the conformal algorithm to produce prediction sets that individually come with conformal coverage guarantees. Then the sets are merged in some way, maintaining an overall guarantee on coverage, which may not be as tight as the individual guarantees, but with the objective of reducing final set size. Methods in this direction include multi-split conformal prediction \citep{SOLARI2022109395}, majority vote merging \cite{gasparin2024merging}, conformal online model aggregation \citep{gasparin2024conformal}, and confidence-level allocation \citep{xu2025aggregating}.

Our work is more closely related to past works on score-level aggregation. \citet{ochoa2025conformal} discuss aggregating multiple scoring functions into a single conformal score, and introduce a particular method for conformal score aggregation (CSA) on classification problems. They establish the central structural observation that aggregating scores preserves more information than aggregating prediction sets. Our contribution is the formalization of the recall-style coverage guarantee setting which is more applicable to NLP tasks. Other relevant score-level works include conformal score averaging \cite{luo2025conformal}, symmetric aggregation \citep{alami2026symmetric}, efficiency/validity first conformal prediction \citep{yang2025selection}, and conformal aggregation of one-shot predictors~\citep{waldron2026caos}.

Conformal prediction can be seen as a type of uncertainty quantification. By applying it, users gain more information about how confident their model's predictions are, as larger sets indicate greater uncertainty \cite{angelopoulos2023conformal}. However, when applying conformal prediction in practice, one must be aware of potential negative interactions with other aspects of trust \citep{cresswell2025trustworthy}. In particular, researchers have considered whether conformal prediction can impact fairness. When datapoints are stratified by group identities, \citet{romano2020with} proposed that ensuring each group receives prediction sets with equal coverage should be a fairness goal \cite{zhou2024conformal}. This is a type of procedural fairness, since it relates to properties of the algorithm itself. However, \citet{gao2026burden} found that there can be inherent tradeoffs between fairness and coverage disparity.

In contrast to the above, \citep{cresswell2024, cresswell2025conformal} proposed that fairness should consider the downstream impacts of conformal prediction, that is, how sets are used, not just how they are produced. They found that equalizing set size, not coverage, resulted in the most fair outcomes on downstream tasks. This was further confirmed in larger scale experiments \citep{liu2026beyond}.

Other works have have borrowed traditional notions of fairness in machine learning to apply them to conformal prediction, such as demographic parity \citep{liu2022fairnessCQR}, Equal Opportunity \citep{wang2024equal}, while \citet{vadlamani2025a} create a framework that supports many definitions. Counterfactual fairness \cite{kusner2017counterfactual} has also been applied in a conformal setting \citep{guldogan2025counterfactual}. Practitioners should thus be aware to evaluate their deployments of conformal prediction for fairness impacts, and employ suitable methods to improve fairness where possible.

\section{Detailed Exchangeability Argument}
\label{sec:app-exchangeability}

This appendix extends the validity argument in Section~\ref{ssec:fusion} for greater clarity.

For fixed-ICL scorers, ICL examples are drawn from the pool $\Pi$, which is disjoint from the calibration and test sets.  Once drawn and added to the LLM's prompt, each relevance scoring function $R_j(c; x)$ is a deterministic function of the input, so exchangeability of calibration and test samples directly implies exchangeability of conformal scores. The ensemble case is immediate: $\bar R(c; x) = \tfrac{1}{K}\sum_j R_j(c; x)$ is also a deterministic function of the input once all $K$ ICL example sets are fixed, and the conformal guarantee applies to any fixed scoring function.

\subsection{Per-Sample ICL Scorers}

For retrieval-based ICL selection (e.g., BM25), the ICL examples depend on the input $x$ via the selection function $f(x, \text{pool})$.  Conditioning on the fixed pool, this is still a deterministic function of $x$, and exchangeability of calibration and test data is preserved.

\textbf{Formal argument (tower property).}  Let $\Pi$ denote the pool draw. For any fixed realization $\pi$ of $\Pi$, the relevance scoring function $R_j(\cdot; \cdot \mid \Pi = \pi)$ is a deterministic function, so the conformal guarantee holds conditionally:
\begin{equation}
  P\bigl[S_\beta \leq \hat{q} \mid \Pi = \pi\bigr]
  \leq \alpha
  \quad \text{for all } \pi.
\end{equation}
By the tower property:
\begin{equation}
  P\bigl[S_\beta \leq \hat{q}\bigr]
  = \mathbb{E}_\Pi\bigl[P[S_\beta \leq \hat{q} \mid \Pi]\bigr]
  \leq \alpha,
\end{equation}
so the guarantee holds unconditionally.

\subsection{What Would Break Exchangeability}

Two scenarios would break the argument:
\begin{enumerate}[leftmargin=*]
  \item \textbf{Data leakage:} drawing ICL examples from the calibration or test set rather than the disjoint pool, which couples the scoring function to the calibration data.
  \item \textbf{Data-driven aggregation:} letting the aggregation function depend on calibration data, introducing drift in the conformal score distribution between calibration and test.
\end{enumerate}
The three-way split (pool drawn first, calibration and test drawn independently) prevents scenario~1. The fixed aggregation prevents scenario~2.

This is precisely \citet{ochoa2025conformal}'s concern: their data-driven quantile envelopes require a two-stage calibration split to restore exchangeability, which our fixed mean ensembling avoids.

\section{Ensemble Theory --- Full Proofs and Additional Results}
\label{sec:app-proofs}

This appendix contains full proofs of all theoretical results stated in Section~\ref{sec:theory}, along with additional results: the ensemble floor inequality, the $\beta < 1$ extension (generalized $\mathrm{Comp}_\beta$ decomposition, sacrifice fragmentation characterization, and the $\beta$ dichotomy), set containment bounds, conciseness guarantees, and variance reduction.

\subsection{Setup and Notation}
\label{sec:app-proofs-notation}

We first recall the notation from \Cref{sec:theory} for convenience. We have $K$ scoring functions $R_1, \ldots, R_K$, each assigning a relevance score in $[0, 1]$ to every content span $c$. The mean ensemble scoring function averages them per-span:
\begin{equation}
  \bar R(c; x)
  = \frac{1}{K} \sum_{j=1}^{K} R_j(c; x).
\end{equation}
The positive set $y^*$ contains $|y^*|$ ground truth relevant content spans. For a given scoring function (either one of the $R_j$ or $\bar R^{(K)}$) and a given $x$, sort the scores for positive spans in ascending order: $R_{(1)} \leq R_{(2)} \leq \cdots \leq R_{(|y^*|)}$.  Then: 
\begin{equation*}
  S_\beta(x, y^*) = R_{(r)} \quad \text{where } r = |y^*| - \lceil \beta |y^*| \rceil + 1.
\end{equation*}
At $\beta = 1$: $r = 1$, so $S(x, y^*) = \min_{c \in y^*} R(c;x)$.

\subsection{Ensemble floor advantage proof}
\label{sec:app-complementarity}

We present the full proof of Lemma~\ref{lem:complementarity} from Section~\ref{ssec:complementarity}.

\smallskip
\decomposition*
\smallskip

\begin{proof}
For convenience, we recall the definition of complementarity from \Cref{def:complementarity},
\begin{multline*}
  \mathrm{Comp}(R_1, R_2)
  = \min_{c \in y^*}\bigl[R_1(c) + R_2(c)\bigr] \\
  - \bigl[S(R_1) + S(R_2)\bigr].
\end{multline*}
We note that non-negativity of $\mathrm{Comp}$ follows from superadditivity of the minimum. The mean ensemble's conformal score is
\begin{equation}
  S^{(2)}
     = \frac{1}{2}\min_{c \in y^*}\bigl[R_1(c) + R_2(c)\bigr].
\end{equation}
Substituting the definition of $\mathrm{Comp}$ directly gives
\begin{align}
  S^{(2)}
  &= \frac{1}{2}[S(R_1) + S(R_2) + \mathrm{Comp}]\nonumber\\
  &= \frac{1}{2}[S_{\max} + S_{\min} + \mathrm{Comp}].
\end{align}
The mean ensemble has an advantage when $S^{(2)} > S_\text{max}$, which occurs exactly when
\begin{equation}
 S^{(2)} - S_{\max} = \frac{1}{2}[\mathrm{Comp} - (S_{\max} - S_{\min})] > 0,
\end{equation}
or $\mathrm{Comp} > S_{\max} - S_{\min}$.
\end{proof}

\begin{remark}[Disjoint minimizers]
Let $A_j = \{c \in y^* : R_j(c) = \min_{c'} R_j(c')\}$ be the set of positive content spans with minimal score according to $R_j$.  If $\bigcap_j A_j = \emptyset$ (no content span is a minimizer for all $R_j$ simultaneously), then $S(\bar R) > \frac{1}{K}\sum_j S(R_j)$ strictly.  This is because for every $c$, since $c \notin \bigcap_j A_j$, there exists some scoring function $R_j$ with $R_j(c) > A_j$ strictly, giving $\bar{R}(c) > \frac{1}{K}\sum_j R_j(a)$ for all $a\in A_j$.
\end{remark}

\paragraph{Worked example.}\label{sec:app-comp-worked-example}
Consider $K = 2$ and $p = 3$ relevant content spans:
\smallskip
\begin{center}
\begin{tabular}{@{}cccc@{}}
  \toprule
  Content span & $R_1$ & $R_2$ & $\bar{R}$ \\
  \midrule
  $c_1$ & 0.8 & 0.6 & 0.70 \\
  $c_2$ & 0.4 & 0.9 & 0.65 \\
  $c_3$ & 0.7 & 0.5 & 0.60 \\
  \bottomrule
\end{tabular}
\end{center}
\smallskip

\noindent
Here $S(R_1) = 0.4$, $S(R_2) = 0.5$, so $S_{\max} = 0.5$ and $S_{\min} = 0.4$.
The complementarity measure is $\mathrm{Comp} = \min\{1.4,\, 1.3,\, 1.2\} - 0.9 = 0.3$, which exceeds the floor gap of $S_{\max}-S_{\min} = 0.1$.
The ensemble achieves $S^{(2)} = 0.60 > 0.50=S_{\max}$: $R_1$'s weakness on $c_2$ is compensated by $R_2$, and vice versa for $c_1$. If both scorers shared the same weakness ($R_2(c_2) = 0.4$), $\mathrm{Comp}$ would merely match the floor gap (both being zero) and the ensemble would not improve.

\subsection{$O(\tfrac{1}{K})$ diminishing returns proof}
\label{sec:app-adding}

We present the full proof of Proposition~\ref{prop:diminishing} from Section~\ref{ssec:diminishing}.

\smallskip
\diminishing*
\smallskip

\begin{proof}
The new ensemble average is:
\begin{equation}
  \bar{R}^{(K+1)}(c)
  = \frac{K  \bar{R}^{(K)}(c) + R_{K+1}(c)}{K+1}.
\end{equation}
To raise the ensemble floor $S^{(K+1)} > S^{(K)}$, we require $\bar{R}^{(K+1)}(c) > S^{(K)}$ for every $c \in y^*$.
Substituting and rearranging:
\begin{align}
  K \cdot \bar{R}^{(K)}(c) + R_{K+1}(c)
  &> (K+1) \cdot S^{(K)} \nonumber\\
  R_{K+1}(c)
  &> S^{(K)} - K \cdot \delta^K(c).
\end{align}

For the converse direction, we simply follow the algebraic manipulations in reverse.

Next we consider the bound in \Cref{eq:diminishing-bound}. Let $c^* =\arg\min_{c \in y^*} \bar{R}^{(K)}(c)$ such that $\bar{R}^{(K)}(c^*) = S^{(K)}$. We then have
\begin{align}
\begin{aligned}
S^{(K+1)} &= \min_{c\in y^*} \bar{R}^{(K+1)}(c)\\ &\le \bar{R}^{(K+1)}(c^*) = \frac{K S^{(K)} + R_{K+1}(c^*)}{K+1}.
\end{aligned}
\end{align}
Since relevance scores are in $[0, 1]$ by definition, $R_{K+1}(c^*) \le 1$:
\begin{equation}
S^{(K+1)} \le \frac{K S^{(K)} + 1}{K+1}.
\end{equation}
Subtracting $S^{(K)}$ from both sides yields the upper bound:
\begin{align}
S^{(K+1)} - S^{(K)} 
&\le \frac{1 - S^{(K)}}{K+1}  = O\left(\frac{1}{K}\right),
\end{align}
where we note that $0\leq S^{(K)} \leq 1$ again by our definitions.
\end{proof}

\subsection{Extension to $\beta<1$}
\label{sec:app-comp-beta}

This subsection formalizes the extension of results from \Cref{sec:theory} to $\beta < 1$, now focusing on the conformal score function $S_\beta(x, y^*)$ from \Cref{eq:s-beta}.
\textbf{Generalized $\mathrm{Comp}_\beta$ decomposition.}  For $\beta < 1$
and $K = 2$ scorers, define
\begin{multline}\label{eq:comp-beta}
  \mathrm{Comp}_\beta(R_1, R_2) \\
  = \mathrm{OrderStat}_r\bigl(\{R_1(c) + R_2(c) : c \in y^*\}\bigr) \\
  \qquad - \bigl[S_\beta(R_1) + S_\beta(R_2)\bigr],
\end{multline}
using the $r$'th order statistic over positive.

\begin{lemma}[General ensemble floor advantage]
\label{prop:comp-beta}
With $S_{\beta,\max} = \max(S_\beta(R_1), S_\beta(R_2))$ and $S_{\beta,\min} = \min(S_\beta(R_1), S_\beta(R_2))$, the mean ensemble floor
\begin{equation}
    S_\beta^{(2)}(x, y^*) = \bar R^{(2)}_{(r)}(c;x)
\end{equation}
improves, $S^{(2)}_\beta > S_{\beta,\max}$, iff $\mathrm{Comp}_\beta > S_{\beta,\max} - S_{\beta,\min}$.
\end{lemma}
\begin{proof}
If we use the $\mathrm{OrderStat}$ notation, the ensemble scoring function is $S_\beta^{(2)} = \tfrac{1}{2}\,\mathrm{OrderStat}_r(\{R_1(c) + R_2(c)\})$.  Substituting the $\mathrm{Comp}_\beta$ definition gives $S_\beta^{(2)} = \tfrac{1}{2}(S_{\beta,\max} + S_{\beta,\min} + \mathrm{Comp}_\beta)$. Hence we immediately find that $S^{(2)}_\beta-S_{\beta,\max} > 0$ exactly when $\mathrm{Comp}_\beta > S_{\beta,\max} - S_{\beta,\min}$.
\end{proof}
\noindent We can equivalently write this complementarity condition as $S_\beta^{(2)} > S_{\beta,\max}$ iff at least $\lceil \beta |y^*| \rceil$ positives satisfy $\bar{R}(c) > S_{\beta,\max}$.
\begin{proof}
$S_\beta^{(2)}$ exceeds $S_{\beta,\max}$ iff at most $r-1$ values are $\leq S_{\beta,\max}$, i.e., at least $|y^*| - r + 1 = \lceil \beta |y^*| \rceil$ values are $> S_{\beta,\max}$.  At $\beta = 1$ this reduces to Lemma~\ref{lem:complementarity}'s ``for all $c \in y^*$''
condition.
\end{proof}

Note that the $\min$ function (order statistic $r = 1$) is concave, so $\mathrm{Comp}_1 \geq 0$ always holds; higher-order statistics ($r > 1$) are not concave, and $\mathrm{Comp}_\beta$ can be strictly negative. Averaging scoring functions at $\beta = 1$ cannot lower the floor, but we have no such guarantee at $\beta < 1$.

\begin{proposition}[General $O(\frac{1}{K})$ diminishing returns]
\label{prop:adding-beta}
Adding $R_{K+1}$ gives $S_\beta^{(K+1)} > S_\beta^{(K)}$ iff at least $\lceil \beta |y^*| \rceil$ positives satisfy $R_{K+1}(c) > S_\beta^{(K)} - K \delta_\beta^K(c)$, with $\delta_\beta^K(c) = \bar{R}^{(K)}(c) - S_\beta^{(K)}$. When this condition holds,  \begin{equation}
0 < S_\beta^{(K+1)} {-} S_\beta^{(K)} \leq \frac{1-S_\beta^{(K)}}{K+1} = O\Big(\frac{1}{K}\Big).
\end{equation}
\end{proposition}
\begin{proof}
Because $S_\beta^{(K)}$ is defined as the $r$-th order statistic of scoring functions on $c\in y^*$, for all $i \in \{1, 2, \ldots, r\}$ with $r\geq 1$, we have
\begin{equation}
\bar{R}^{(K)}_{(i)}=\bar{R}^{(K)}(c_{(i)}) \leq S_\beta^{(K)}
\end{equation}
Under the expanded ensemble $\bar{R}^{(K+1)}$, the result for any $i\leq r$ is
\begin{equation}
\bar{R}^{(K+1)}(c_{(i)}) = \frac{K \bar{R}^{(K)}(c_{(i)}) + R_{K+1}(c_{(i)})}{K+1}.
\end{equation}
Using $R_{K+1}(c) \le 1$ (since $R_{K+1}(c)\in [0, 1]$ by definition),
\begin{equation}
\bar{R}^{(K+1)}(c_{(i)}) \le \frac{K S_\beta^{(K)} + 1}{K+1}.
\end{equation}
Because all $r$ spans in $\{c_{(1)}, \ldots, c_{(r)}\}$ satisfy this inequality, there exist at least $r$ positive content spans whose new ensemble scores do not exceed $\frac{K S_\beta^{(K)} + 1}{K+1}$. By definition of the $r$-th order statistic, the $r$-th smallest score $S_\beta^{(K+1)}$ cannot exceed the maximum score among any subset of $r$ elements. Thus,
\begin{equation}
S_\beta^{(K+1)} \le \frac{K S_\beta^{(K)} + 1}{K+1}.
\end{equation}
Subtracting $S_\beta^{(K)}$ from both sides yields
\begin{equation}
S_\beta^{(K+1)} - S_\beta^{(K)} \le \frac{1 - S_\beta^{(K)}}{K+1}.
\end{equation}
\end{proof}

\section{Reproducibility Details}
\label{sec:app-reproducibility}

This appendix documents the settings required to reproduce our pipeline: model, sampling temperature, random seed, prompt templates, data splits, and calibration size.

\paragraph{Model and inference settings.}
All experiments use Gemini-2.5-Flash-Lite \citep{google2025gemini25flash}, accessed January--April 2026) at sampling temperature $T{=}0$. The random seed is fixed for ICL pool sampling, stratified selection, DPP subset sampling, negative sub-sampling in contrastive formatting, and calibration/test splits. Scoring and retry prompts are given verbatim in Appendix~\ref{sec:app-contrastive} and Appendix~\ref{sec:app-prompts}.

\paragraph{ICL embedding model.}
Both embedding-based ICL selection strategies (\texttt{anchor\_dpp}, \texttt{pattern\_dpp}) use all-MiniLM-L6-v2 \citep{reimers2019sentencebert} via the \texttt{sentence-transformers} library. This 22M-parameter model produces 384-dimensional embeddings.

\paragraph{Data splits and sample counts.}
Data split sizes and per-dataset sample counts are reported in Table~\ref{tab:datasets}. Calibration size is fixed at $n=100$ across all datasets

\paragraph{Pool and calibration sizing rationale.}
The ICL pool size is set to 50 for single-intent datasets, or 20 per intent for multi-intent datasets (see \Cref{tab:datasets}). With $k{=}2$ examples per sub-scoring function and $K{=}4$ sub-scoring functions, only 8 pool examples are consumed per test sample, but a larger pool lets the diversity-maximizing strategies (\texttt{anchor\_dpp}, \texttt{pattern\_dpp}) select from a richer candidate set per input, improving complementarity.  A calibration size of $n=100$ balances conformal threshold variance against annotation cost, and gives reasonably precise coverage estimates in our experiments. Both sizes are fixed across all datasets, making the budget requirement predictable for practitioners.

\paragraph{Code availability.}
All experiment code, ICL strategy implementations, and data processing scripts are available at \href{https://github.com/layer6ai-labs/conformal-relevance}{github.com/layer6ai-labs/conformal-relevance}.

\subsection{Dataset Reformulations}
\label{sec:app-datasets}

Four of the seven datasets have their native task reformulated to per-sentence binary relevance. The mapping from native task to sentence-level positive definition is:

\begin{itemize}[leftmargin=*,itemsep=2pt,topsep=2pt]
  \item \textbf{HotpotQA} \citep{yang2018hotpotqa} --- Multi-hop QA with supporting-fact annotations. Sentence-level relevance: a context sentence is positive when it is annotated as a supporting fact for the question. The question text fills the \texttt{intent} field.
  \item \textbf{ECTSum} \cite{mukherjee-etal-2022-ectsum} --- Abstractive financial-call summarization. Sentence-level relevance: a transcript sentence is positive when it is included in the human-written summary (matched by ROUGE-based alignment). No intent / single-intent dataset.
  \item \textbf{ContractNLI} \cite{koreeda2021contractnli} --- Contract NLI (entailment/contradiction/neutral hypothesis classification). Sentence-level relevance: a contract span is positive when it appears in the gold evidence set for the (contract, hypothesis) pair. The 17 hypotheses act as intents (stratified ICL); only Entailment + Contradiction pairs are retained as usable.
  \item \textbf{Evidence Inference} \cite{deyoung2020evidence} --- Clinical-trial outcome inference. Sentence-level relevance: an article sentence is positive when it is annotated as an evidence span for the (intervention, comparator, outcome) prompt. The 8 user splits act as intents (stratified ICL).
\end{itemize}

The remaining three datasets (SubSumE, PUMA, PhysioNet) come with sentence-level relevance labels; no reformulation is applied. PhysioNet sentences come from the PhysioNet Deidentified Medical Text v1.0 corpus \citep{PhysioNet-deidentifiedmedicaltext-1.0,neamatullah2008automated}, and per-sentence PHI labels follow the i2b2 2014 deidentification schema \citep{stubbs2015automated} applied on top of the corpus.

\vspace{-2pt}
\section{Prompting and ICL Specification}
\label{sec:prompting_specification}
\vspace{-2pt}
In this appendix, we provide information on prompts and other details on our method.
\vspace{-2pt}
\subsection{Contrastive ICL Format Specification}
\label{sec:app-contrastive}
\vspace{-2pt}
This appendix provides the full contrastive ICL format specification used to present labeled examples to the LLM, the scoring prompt template, and the retry logic for handling malformed LLM outputs. The contrastive format lists positive sentences (``Sentences selected'') and a balanced sample of negative sentences (``Sentences NOT selected''), with a minimum of two negatives.

\paragraph{Scoring prompt template.}
The base prompt instructs the LLM to infer a selection distinction from the ICL block and apply it to the evaluation section. Placeholders: \texttt{\{icl\_examples\}} is replaced by the formatted ICL block; \texttt{\{task\_line\}} is an optional one-line hint prepended for domain-specific tasks (empty string otherwise); \texttt{\{intent\}} is the per-sample query or perspective; \texttt{\{sentences\}} is the numbered sentence list. For no-intent datasets (ECTSum, PhysioNet), the \texttt{Intent:} line is omitted entirely and a variant template is used.

\begin{Verbatim}[fontsize=\scriptsize,breaklines=true,breakanywhere=true]
In the examples below, some sentences were selected and others were not. Identify what distinguishes the selected sentences from the non-selected ones. Then apply that SAME distinction to score each sentence in the evaluation section.

Each score should be a two decimal float between 0 and 1. A sentence that clearly matches the demonstrated distinction should score close to 1 (> 0.8). A sentence that does not match should score close to 0 (< 0.2). Use intermediate scores (0.3, 0.5, 0.7) for partial matches.

IMPORTANT: Output must be valid JSON format with sentence indices as keys.

{icl_examples}
Now evaluate the following:

{task_line}Intent: {intent}
Sentences to evaluate:
{sentences}

Scores (JSON format):
\end{Verbatim}

\paragraph{Contrastive ICL example format.}
Each pool example is formatted as: the full sentence list (numbered 0-indexed), then the selected positives as a comma-separated list of quoted sentence texts, then a balanced sample of negatives under the ``NOT selected'' header. Negatives are sub-sampled to achieve a near-1:1 ratio with positives, with a minimum of two negatives; this is capped by the number of available negatives. Sub-sampling uses a per-example deterministic seed derived from the global seed. Multiple examples are concatenated with double newlines and numbered (\texttt{Example 0:}, \texttt{Example 1:}, \ldots), separated by ``\texttt{---}'' dividers.

An example block for an intent-bearing dataset looks as follows (shortened for illustration):

\begin{Verbatim}[fontsize=\scriptsize,breaklines=true,breakanywhere=true]
Example 0:
Intent: History
Sentences to evaluate:
0. The city was founded in 1832.
1. The annual rainfall averages 800 mm.
2. It was incorporated as a borough in 1901.
3. The main industry is tourism.
Sentences selected for "History": "The city was founded in 1832.", "It was incorporated as a borough in 1901." 
Sentences NOT selected for "History" (examples): "The annual rainfall averages 800 mm.", "The main industry is tourism."
---
\end{Verbatim}

For no-intent datasets, the \texttt{Intent:} line is omitted and the labels read \texttt{Sentences selected:} and \texttt{Sentences NOT selected (examples):}.

\paragraph{Retry prompt.}
When the LLM response omits sentence indices, the missing sentences are re-submitted using the retry template below. The same ICL block is included to maintain consistency with the original scoring context.

\begin{Verbatim}[fontsize=\scriptsize,breaklines=true,breakanywhere=true]
You previously scored sentences but missed some. In the examples below, some sentences were selected and others were not. Identify what distinguishes the selected sentences from the non-selected ones. Then apply that SAME distinction to score each MISSED sentence below.

Each score should be a two decimal float between 0 and 1. A sentence that clearly matches the demonstrated distinction should score close to 1 (> 0.8). A sentence that does not match should score close to 0 (< 0.2). Use intermediate scores (0.3, 0.5, 0.7) for partial matches.

IMPORTANT: Output must be valid JSON format with sentence indices as keys.

{icl_examples}
Now evaluate the following:

{task_line} Intent: {intent}
Missing sentences:
{missing_sentences}

Scores (JSON format with ONLY the missing sentence numbers as keys):
\end{Verbatim}

\paragraph{Format as controlled variable.}
The contrastive format is held constant across all experimental conditions: ICL-free baselines (ICL0), individual strategies, and ensembles. All comparisons in Table~\ref{tab:main} (main result) and Table~\ref{tab:app-composition} (composition ablation) are therefore measured within the same format, and ensemble gains over single strategies are attributable to strategy diversity, not prompt formatting. A format--strategy interaction is theoretically possible: different strategies select examples with different positive/negative ratios, and the contrastive sub-sampling handles these differently. However, such an interaction would need to \emph{systematically} favor diverse strategy sets over individual strategies across all seven datasets---a strong and specific assumption unsupported by any proposed mechanism. Preliminary experiments during development confirmed that the contrastive layout outperforms both positive-only (no negatives shown) and full-document (all sentences labeled, no sub-sampling) alternatives for individual strategies; a full format $\times$ ensemble factorial ablation is left to future work.

\subsection{Windowed Context Algorithm and Ablation}
\label{sec:app-windowed}

This appendix describes the windowed ICL compression algorithm used for long-document datasets and presents ablation results comparing windowed and full ICL on SubSumE (average 461 sentences per document). The algorithm retains all positive sentences, asymmetric random-width context windows around each positive, and a sample of remote negatives, substantially reducing prompt length while preserving ICL quality.

\paragraph{Algorithm.}
Windowed context is applied automatically when a dataset's average document length exceeds 30 sentences or its average positive rate falls below 15\%; both SubSumE and Evidence Inference trigger this condition, while PUMA, ECTSum, PhysioNet, HotpotQA, and ContractNLI do not. For each ICL pool example selected by a retrieval strategy, the full sentence list is replaced by a compressed version retaining: (i)~all positive sentences; (ii)~a local context window around each positive with independently sampled width in each direction (left width $\sim \mathcal{U}(1, w_L)$, right width $\sim \mathcal{U}(1, w_R)$, with $w_L = w_R = w$); and (iii)~a uniform random subsample of sentences outside all windows (``remote negatives''), with subsample size set to match the number of context-window negatives. The context-window half-width $w$ is fixed at $w = 2$ throughout. For test documents, no windowing is applied: the full sentence list is scored in a single LLM call.

\section{ICL Selection Strategy Details}
\label{sec:app-strategies}
\vspace{-2pt}
This appendix provides detailed algorithmic descriptions for the four ICL selection strategies $\sigma_j$ used in the Ens4 ensemble: \texttt{anchor\_dpp} (centroid anchor with conditional DPP), \texttt{pattern\_dpp} (DPP over relevance-direction vectors), \texttt{bm25} (lexical retrieval with the BM25 algorithm), and \texttt{random} (uniform sampling from the ICL). We include DPP kernel formulations, embedding model specifications, and BM25 parameter settings.
\vspace{-2pt}
\paragraph{Anchor-DPP (\texttt{anchor\_dpp}).}
A two-phase selection procedure that combines centroid proximity with conditional diversity. First, the query document $x$ is embedded via a language embedding model (see App.~\ref{sec:app-reproducibility}). In Phase~1, the example document from the ICL pool which has embedding nearest to the query embedding as measured by cosine similarity is selected as the anchor, and included in the final ICL set. In Phase~2, the remaining $k - 1$ examples are drawn from a conditional Determinantal Point Process (DPP) whose kernel is the cosine Gram matrix of pool embeddings, ridge-regularized to ensure positive definiteness. Conditioning on the anchor biases the sample toward examples that are both semantically close to the query document and mutually diverse relative to it. This strategy targets ICL pool examples that are representative of the test domain while remaining internally diverse, reducing the chance that all $k$ examples provide redundant information to the LLM scorer.
\vspace{-2pt}
\paragraph{Pattern-DPP (\texttt{pattern\_dpp}).}
This DPP operates in the space of sentence-level embeddings rather than full-document embeddings. For each ICL pool example, a ``relevance direction'' is computed as the difference between the centroid embedding of its positive spans ($c\in y^*$) and the centroid embedding of its negative spans ($c\not\in y^*$). The relevance direction captures the specific distinction of content relevance the document exemplifies. The $k$ ICL examples are then selected by running a DPP over the L2-normalized relevance directions to maximize diversity. Because this strategy selects examples whose relevance directions are maximally orthogonal, the resulting scoring function captures a different cross-section of the relevance signal than \texttt{anchor\_dpp}.
\vspace{-2pt}
\paragraph{BM25 (\texttt{bm25}).}
This strategy uses lexical retrieval with the Okapi BM25 algorithm \citep{robertson2009bm25}.
For each intent within the ICL pool, or the full pool for uni-intent datasets, a BM25 index is built over the concatenated text of documents, tokenized by whitespace. At scoring time, the query document $x$ is used as a query to the index and the $k$ highest-scoring documents are retrieved. BM25 captures lexical overlap that semantic embeddings may miss---recurring domain-specific terminology, abbreviations, or proper nouns---and is entirely independent of the embedding model used by \texttt{anchor\_dpp} and \texttt{pattern\_dpp}, making it a structurally distinct signal within the ensemble.
\vspace{-2pt}
\paragraph{Random (\texttt{random}).}
This strategy samples documents uniformly at random without replacement from the intent-stratified pool. Given a query document $x$, $k$ examples are drawn from the ICL pool $\Pi$ using a fixed seed, and independently of the query document's content or embedding. This strategy introduces no retrieval bias and serves as an ensemble regularizer: its scores are uncorrelated with any document-specific signal, so averaging with embedding-based scorers dampens their shared biases rather than reinforcing them.

\vspace{-2pt}
\section{ICL0 Prompts}
\label{sec:app-prompts}
\vspace{-2pt}
This appendix contains the task-specific prompt templates used for the ICL0 baseline. The ECTSum prompt was adapted from the Importance Scoring prompt of \citet{kuwahara2025conformal}, with minor modifications for sentence-level extraction and output formatting; prompts for the remaining tasks followed the same structure with manually written task descriptions. For six of seven datasets, the prompts were written once without empirical tuning on labeled data. For Evidence Inference, candidate variants were evaluated on the 100-sample calibration set, and the selected prompt was evaluated once on the untouched test set.

\subsection{ECTSum}
\vspace{-2pt}
Financial earnings call summarization. The scoring function evaluates sentence-level alignment with the overall message of the transcript. No intent placeholder; the task is fully specified by the prompt framing.

\begin{Verbatim}[fontsize=\scriptsize,breaklines=true,breakanywhere=true]
Evaluate the importance of each input sentence in the original text, based on how the information carried in the sentence is aligned with the overall message. and provide a importance score for EACH input sentence. Each output score should be a two decimal float number ranged between 0 and 1, indicating how important the corresponding input sentence is in the context of the text document. For example, if sentence 1's information is highly aligned with that of the input text, and very likely to be included in the summary, then score 1 should be close to 1, say greater than 0.8; if information carried in sentence 3 is trivial or only remotely related to the central message of the text, and is not worthy of inclusion in the summary, then score 3 should be close to 0, say less than 0.2.

IMPORTANT: Output must be valid JSON format with sentence indices as keys.

Sentences to evaluate:
{sentences}

Importance scores (JSON format):
\end{Verbatim}

\subsection{SubSumE}
\vspace{-2pt}
Query-focused summarization over Wikipedia articles. \texttt{\{intent\}}
is filled at runtime with one of 10 query intents (e.g., ``History'',
``Geography'').

\begin{Verbatim}[fontsize=\scriptsize,breaklines=true,breakanywhere=true]
Evaluate each sentence based on how well it answers or provides information relevant to the given query. Provide a score for EACH sentence. Each score should be a two decimal float between 0 and 1.

For example, if a sentence directly answers or supports the query, its score should be close to 1 (greater than 0.8); if a sentence is unrelated to the query or provides only irrelevant background, its score should be close to 0 (less than 0.2).

IMPORTANT: Output must be valid JSON format with sentence indices as keys.

Query: {intent}
Sentences to evaluate:
{sentences}

Scores (JSON format):
\end{Verbatim}
\subsection{PUMA}
\vspace{-2pt}
Perspective-based QA over medical discussion threads. \texttt{\{intent\}} is filled with one of 5 perspective labels (e.g., ``Information'', ``Suggestion'').

\begin{Verbatim}[fontsize=\scriptsize,breaklines=true,breakanywhere=true]
Evaluate each sentence based on how well it expresses or addresses the given perspective in this discussion. Provide a score for EACH sentence. Each score should be a two decimal float between 0 and 1.

For example, if a sentence directly addresses the {intent} perspective, its score should be close to 1 (greater than 0.8); if a sentence is unrelated to the {intent} perspective or addresses a different aspect, its score should be close to 0 (less than 0.2).

IMPORTANT: Output must be valid JSON format with sentence indices as keys.

Perspective: {intent}
Sentences to evaluate:
{sentences}

Scores (JSON format):
\end{Verbatim}
\subsection{PhysioNet}
\vspace{-2pt}
PHI entity detection in de-identified clinical notes. No intent placeholder; PHI categories are enumerated directly in the prompt.

\begin{Verbatim}[fontsize=\scriptsize,breaklines=true,breakanywhere=true]
Evaluate each sentence for the presence of Protected Health Information (PHI). PHI includes: patient names, doctor names, dates (birth, admission, discharge, specific times), locations (cities, hospitals, addresses), medical record numbers, phone numbers, and ages. Provide a PHI score for EACH sentence. Each score should be a two decimal float between 0 and 1.

For example, if a sentence contains clear PHI such as a person's name, specific date, or location, its score should be close to 1 (greater than 0.8); if a sentence contains only general medical information with no personal identifiers, its score should be close to 0 (less than 0.2).

IMPORTANT: Output must be valid JSON format with sentence indices as keys.

Sentences to evaluate:
{sentences}

PHI scores (JSON format):
\end{Verbatim}
\subsection{HotpotQA}
\vspace{-2pt}
Multi-hop QA supporting fact identification.  \texttt{\{intent\}} is the per-sample question (each sample has a unique question, so no intent stratification is applied).

\begin{Verbatim}[fontsize=\scriptsize,breaklines=true,breakanywhere=true] 
Evaluate each sentence based on how useful it is as a supporting fact for answering the given question. Provide a score for EACH sentence. Each score should be a two decimal float between 0 and 1.

For example, if a sentence provides information needed to answer the question, its score should be close to 1 (greater than 0.8); if a sentence provides unrelated information that doesn't help answer the question, its score should be close to 0 (less than 0.2).

IMPORTANT: Output must be valid JSON format with sentence indices as keys.

Question: {intent}
Sentences to evaluate:
{sentences}

Scores (JSON format):
\end{Verbatim}
\subsection{Evidence Inference}
\vspace{-2pt}
Statistical evidence extraction from clinical trial reports. \texttt{\{intent\}} is one of 8 annotator perspective identifiers (\texttt{User0}--\texttt{User7}).

\begin{Verbatim}[fontsize=\scriptsize,breaklines=true,breakanywhere=true]
From {intent}'s perspective, evaluate each sentence for its value as statistical evidence of the treatment's effect on the outcome. Evidence sentences typically contain: statistical results (p-values, confidence intervals), outcome measurements, comparative data between treatment and control groups, hazard ratios, or effect sizes. Provide an evidence score for EACH sentence. Each score should be a two decimal float between 0 and 1.

For example, if a sentence reports statistical results or quantitative treatment outcomes, its score should be close to 1 (greater than 0.8); if a sentence contains only background information, methodology description, or non-quantitative text, its score should be close to 0 (less than 0.2).

IMPORTANT: Output must be valid JSON with ALL sentence indices as keys. Every sentence must be scored and every index must appear.

Perspective: {intent}
Sentences to evaluate:
{sentences}

Evidence scores (JSON format):
\end{Verbatim}
\subsection{ContractNLI}
\vspace{-2pt}
Legal NDA clause relevance scoring. \texttt{\{intent\}} is one of 17 natural language inference hypotheses (e.g., ``Receiving Party shall not use any Confidential Information for any purpose other than the Authorized Purpose'').

\begin{Verbatim}[fontsize=\scriptsize,breaklines=true,breakanywhere=true]
Evaluate each contract clause (sentence or list item) based on whether it provides evidence for or against the given hypothesis about this NDA. A clause is evidence if it directly supports, contradicts, or qualifies the hypothesis. Provide a relevance score for EACH clause. Each score should be a two decimal float between 0 and 1.

For example, if a clause directly addresses the hypothesis (e.g., explicitly states an obligation or right that the hypothesis describes), its score should be close to 1 (greater than 0.8); if a clause is about unrelated obligations, definitions, or boilerplate unrelated to the hypothesis, its score should be close to 0 (less than 0.2).

IMPORTANT: Output must be valid JSON format with sentence indices as keys.

Hypothesis: {intent}
Contract clauses to evaluate:
{sentences}

Relevance scores (JSON format):
\end{Verbatim}

\begin{table*}[t]
\centering
\caption{Main empirical results for MAP comparing all ICL selection strategies. \textbf{Bold} marks the best configuration per row. Avg Single is the mean MAP across the four strategies.}
\label{tab:app-full-matrix}
\small
\begin{tabular}{lccccccc}
\toprule
Dataset & ICL0 & \rotatebox{70}{\texttt{anchor\_dpp}} & \rotatebox{70}{\texttt{bm25}} & \rotatebox{70}{\texttt{pattern\_dpp}} & \rotatebox{70}{\texttt{random}} & Avg Single & Ens4 \\
\midrule
ECTSum          & 0.350 & 0.437 & 0.368 & 0.426 & 0.339 & 0.393 &  \textbf{0.516} \\
Evidence Inf.   & 0.206 & 0.205 & 0.199 & 0.196 & 0.203 & 0.201 &  \textbf{0.304} \\
HotpotQA        & 0.749 & 0.749 & 0.730 & 0.739 & 0.741 & 0.740 &  \textbf{0.839} \\
PhysioNet       & 0.764 & 0.767 & 0.764 & 0.749 & 0.782 & 0.766 &  \textbf{0.879} \\
PUMA            & 0.648 & 0.777 & 0.737 & 0.741 & 0.751 & 0.752 &  \textbf{0.814} \\
SubSumE         & 0.289 & 0.356 & 0.364 & 0.354 & 0.340 & 0.354 &  \textbf{0.464} \\
ContractNLI     & 0.718 & 0.706 & 0.713 & 0.717 & 0.712 & 0.712 &  \textbf{0.828} \\
\bottomrule
\end{tabular}
\end{table*}

\begin{figure*}[t]
\centering
\includegraphics[width=0.7\linewidth]{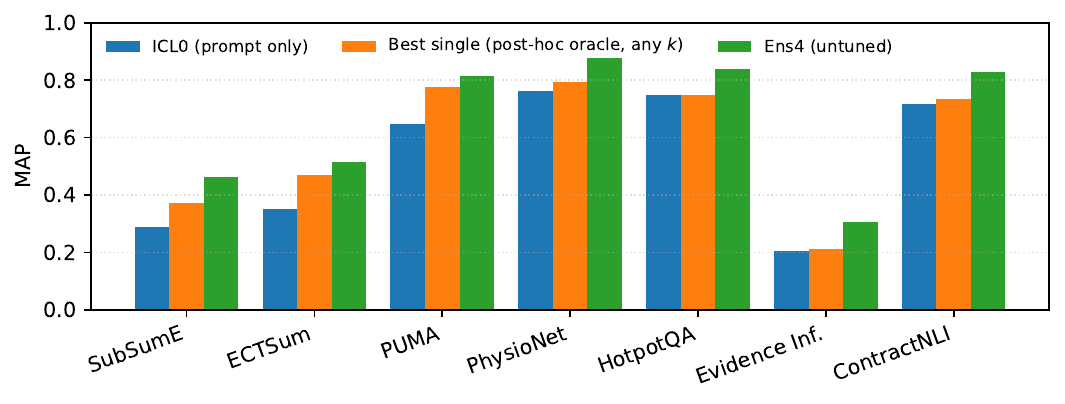}
\caption{MAP comparison across seven datasets:  Ens4 improves over both baselines on all 7 datasets. Numerical values are in Table~\ref{tab:main}.}
\label{fig:main-results}
\end{figure*}

\FloatBarrier

\section{Main Empirical Results}
\label{sec:app-strategy-matrix}

\Cref{tab:app-full-matrix} reports MAP for each of the four Ens4 ICL selection strategies and the Ens4 ensemble at $k \in \{1, 2, 3, 5, 8\}$, with $K{=}4$ fixed for Ens4 runs. Figure~\ref{fig:main-results} visualizes the MAP comparisons from Table~\ref{tab:main}.

\paragraph{Computational cost.}
Ens4 issues $K{=}4$ scoring-LLM calls per document versus ICL0's single call, independent of dataset, model, or document length. There is also an ICL-selection overhead, run once per document on the labeled pool, which is negligible compared to the LLM API call latencies. Each Ens4 call uses the same template prompt as ICL0 augmented with $k{=}2$ contrastive examples (Appendix~\ref{sec:app-contrastive}), so input tokens decompose as
\[
\underbrace{T_{\mathrm{system}} + T_{\mathrm{hint}}}_{\text{shared with ICL0}}
+ \underbrace{2k\cdot T_{\mathrm{ex}}}_{\text{ICL block}}
+ T_{\mathrm{doc}},
\]
where the ICL block is windowed to $\pm 2$ sentences on long-document datasets (Appendix~\ref{sec:app-windowed}) and output tokens are identical across ICL0 and every Ens4 sub-call.  The input-token cost ratio
\begin{multline*}
\frac{4\,(T_{\mathrm{shared}} + 2k\cdot T_{\mathrm{ex}} + T_{\mathrm{doc}})}
     {T_{\mathrm{shared}} + T_{\mathrm{doc}}} \\
= 4\!\left(1 + \frac{2k\cdot T_{\mathrm{ex}}}{T_{\mathrm{shared}} + T_{\mathrm{doc}}}\right)
\end{multline*}
approaches $4{\times}$ from above as $T_{\mathrm{doc}}$ grows, so the multiplier is closest to $4{\times}$ on long-document datasets.

Wall-clock latency is harder to isolate because it also depends on provider-side effects such as throttling and retries. We benchmark with local caching disabled so that every request goes to the LLM's API, using 150 samples per configuration, identical hardware, and serial execution, even though Ens4's calls are embarrassingly parallel. As expected, Ens4 is more expensive than ICL0, but the resulting measurements in \Cref{tab:app-run-ids} should be interpreted as directionally indicative rather than as intrinsic latency multipliers. In practice, the four Ens4 scoring calls are independent and can be parallelized.

\begin{table}[h]
\centering
\caption{Latency for ICL0 and Ens4 with serial execution of LLM API calls. The multiplier is Ens4 relative to ICL0. Measurements include provider-side retries and should be interpreted as directionally indicative.}
\label{tab:app-run-ids}
\small
\setlength{\tabcolsep}{2pt}
\begin{tabular}{lrrr}
\toprule
Dataset & ICL0 (s/sample) & Ens4 (s/sample) & Multiplier \\
\midrule
ECTSum        & 1.64 & 8.03  & 4.9$\times$ \\
Evidence Inf. & 0.76 & 7.69  & 10.1$\times$ \\
HotpotQA      & 0.99 & 7.61  & 7.7$\times$ \\
PhysioNet     & 0.17 & 1.78  & 10.7$\times$ \\
PUMA          & 1.83 & 4.16  & 2.3$\times$ \\
SubSumE       & 2.51 & 45.19 & 18.0$\times$ \\
ContractNLI   & 1.46 & 14.03 & 9.6$\times$ \\
\bottomrule
\end{tabular}
\end{table}

\FloatBarrier

\begin{table*}[t]
\caption{MAP for all non-singleton subsets of the four ICL selection strategies at $k{=}2$; the oracle best single strategy ($K{=}1$) is shown as a reference. \textbf{Bold} marks the best subset per $K$.  $\overline{\mathrm{MAP}}$ is the mean across 7 datasets. Strategy abbreviations: a\,=\,\texttt{anchor\_dpp},  p\,=\,\texttt{pattern\_dpp}, b\,=\,\texttt{bm25}, r\,=\,\texttt{random}.}
\label{tab:app-composition}
\centering
\small
\begin{tabular}{clccccccc c}
\toprule
$K$ & Subset & SubS & ECT & PUMA & Phys & HotQ & EviI & CNLI & $\overline{\mathrm{MAP}}$ \\
\midrule
1 & Best Single${}^\dagger$ & 0.373 & 0.469 & 0.777 & 0.795 & 0.749 & 0.213 & 0.733 & 0.587 \\
\midrule
2 & \textbf{a+p} & 0.417 & 0.492 & 0.797 & 0.842 & 0.812 & 0.260 & 0.786 & 0.629 \\
2 & a+b          & 0.426 & 0.482 & 0.796 & 0.837 & 0.810 & 0.261 & 0.787 & 0.628 \\
2 & a+r          & 0.420 & 0.474 & 0.800 & 0.843 & 0.806 & 0.259 & 0.780 & 0.626 \\
2 & p+b          & 0.424 & 0.473 & 0.780 & 0.837 & 0.804 & 0.255 & 0.787 & 0.623 \\
2 & b+r          & 0.427 &0.431 & 0.785 & 0.855 & 0.797 & 0.253 & 0.785 & 0.619 \\
2 & p+r          & 0.421 & 0.468 & 0.787 & 0.854 & 0.800 & 0.249 & 0.786 & 0.623 \\
\midrule
3 & \textbf{a+p+b} & 0.450 & 0.509 & 0.807 & 0.867 & 0.832 & 0.288 & 0.816 & 0.652 \\
3 & a+p+r        & 0.445 & 0.506 & 0.808 & 0.872 & 0.830 & 0.285 & 0.813 & 0.651 \\
3 & a+b+r        & 0.450 & 0.492 & 0.807 & 0.868 & 0.827 & 0.287 & 0.814 & 0.649 \\
3 & p+b+r        & 0.452 & 0.488 & 0.800 & 0.874 & 0.824 & 0.280 & 0.814 & 0.647 \\
\midrule
4 & \textbf{a+p+b+r (Ens4)} & \textbf{0.464} & \textbf{0.516} & \textbf{0.814} & \textbf{0.879} & \textbf{0.839} & \textbf{0.304} & \textbf{0.828} & \textbf{0.664} \\
\bottomrule
\multicolumn{10}{l}{\footnotesize ${}^\dagger$Oracle post-hoc best single strategy per dataset (varies by dataset).}
\end{tabular}
\end{table*}

\begin{table*}[t]
\centering
\caption{MAP for each individual sub-strategy at $k{=}2$ and $8$. Mean $\Delta$ is the average change from $k{=}2$ to $k{=}8$.}
\label{tab:app-k-scaling-strategies}
\small
\begin{tabular}{lrrrrrrrr rr}
\toprule
& \multicolumn{2}{c}{\texttt{anchor\_dpp}} & \multicolumn{2}{c}{\texttt{bm25}} & \multicolumn{2}{c}{\texttt{pattern\_dpp}} & \multicolumn{2}{c}{\texttt{random}} \\
\cmidrule(lr){2-3} \cmidrule(lr){4-5} \cmidrule(lr){6-7} \cmidrule(lr){8-9}
Dataset & $k{=}2$ & $k{=}8$ & $k{=}2$ & $k{=}8$ & $k{=}2$ & $k{=}8$ & $k{=}2$ & $k{=}8$ & Ens4 \\
\midrule
ECTSum        & 0.437 & 0.397 & 0.368 & 0.371 & 0.426 & 0.339 & 0.339 & 0.392 & 0.516\\
SubSumE       & 0.356 & 0.347 & 0.364 & 0.369 & 0.354 & 0.359 & 0.340 & 0.361 & 0.304\\
HotpotQA      & 0.749 & 0.742 & 0.730 & 0.719 & 0.739 & 0.722 & 0.741 & 0.719 & 0.839\\
PhysioNet     & 0.767 & 0.702 & 0.764 & 0.735 & 0.749 & 0.739 & 0.782 & 0.784 & 0.879\\
PUMA          & 0.777 & 0.737 & 0.737 & 0.731 & 0.741 & 0.744 & 0.751 & 0.731 & 0.814\\
Evidence Inf. & 0.205 & 0.184 & 0.199 & 0.186 & 0.196 & 0.191 & 0.203 & 0.186 & 0.464\\
ContractNLI   & 0.706 & 0.697 & 0.713 & 0.711 & 0.717 & 0.703 & 0.712 & 0.696 & 0.828\\
\midrule
Mean $\Delta$ & \multicolumn{2}{c}{$-0.027$} & \multicolumn{2}{c}{$-0.008$} & \multicolumn{2}{c}{$-0.018$} & \multicolumn{2}{c}{$+0.000$} \\
\bottomrule
\end{tabular}
\end{table*}

\FloatBarrier

\begin{figure}[t]
\centering
\includegraphics[width=0.65\linewidth]{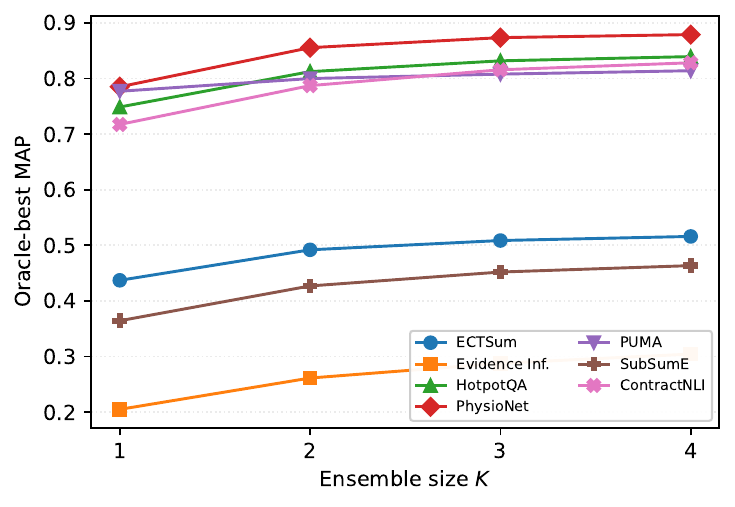}
\caption{Best MAP at each ensemble size $K \in \{1,2,3,4\}$ across all   7 datasets.  All 7 datasets show monotonically increasing, flattening curves, consistent with Proposition~\ref{prop:diminishing}'s intuition.}
\label{fig:oracle-best-k-scaling}
\end{figure}

\section{Robustness Evaluations}\label{sec:app-robustness}
\vspace{-40pt}
\subsection{Ensemble Size ($K$) and Composition}
\label{sec:app-composition}
The results in \Cref{sec:ablations} show the main outcomes. Here, Table~\ref{tab:app-composition} gives the full breakdown of MAP for each non-singleton strategy combination over the seven datasets. Ens4 ranks first among all 11 subsets on every dataset. Mean MAP rises from $0.587$ ($K{=}1$, oracle best single) to $0.625$ ($K{=}2$) to $0.650$ ($K{=}3$) to $0.664$ ($K{=}4$), with strictly diminishing marginal gains, consistent with Proposition~\ref{prop:diminishing}'s intuition. Figure~\ref{fig:oracle-best-k-scaling} visualizes the per-dataset best MAP at each $K \in \{1,2,3,4\}$.

\subsection{Complementarity vs. ICL Quantity}\label{sec:app-quantity}

The default Ens4 configuration uses $K \cdot k = 8$ ICL examples for each query $x$, spread out over 4 scoring functions. This can be directly compared to Best Single at $k{=}8$, using the entire ICL budget in a single retrieval signal, as we show in Table~\ref{tab:app-k-scaling-strategies}. Ens4 outperforms by $+0.070$ to $+0.119$ MAP on every dataset. In fact, increasing the number of ICL examples from $k=2$ to $8$ for a single strategy is not helpful.

\subsection{Complementarity vs.\ Labeled Data.}
\label{sec:app-classifier}

Compared to ICL0, Ens4 uses additional labeled data as ICL examples. To assess whether the additional labeled examples alone explain
Ens4's gains, we train a logistic regression classifier using
all-MiniLM-L6-v2 sentence and intent embeddings, the same embedding model used for DPP retrieval in Ens4. The classifier is
trained on the entire ICL pool together with the 100 calibration samples. This gives an advantage to the classifier, since each scorer in the Ens4 ensemble sees only two examples, for a maximum of eight distinct ICL examples used at a time.

Ens4 outperforms the supervised classifier on 6/7 datasets.
The classifier performs best only on ECTSum. Results are shown in \Cref{tab:app-classifier}.

\begin{table}[t]
\centering
\caption{MAP comparison between the supervised classifier baseline
and Ens4. Bold indicates the better method.}
\small
\label{tab:app-classifier}
\begin{tabular}{lcc}
\toprule
Dataset & Classifier & Ens4 \\
\midrule
ECTSum            & \textbf{0.615} & 0.516 \\
Evidence Inf.     & 0.281 & \textbf{0.304} \\
HotpotQA          & 0.281 & \textbf{0.839} \\
PhysioNet         & 0.732 & \textbf{0.879} \\
PUMA              & 0.574 & \textbf{0.814} \\
SubSumE           & 0.322 & \textbf{0.464} \\
ContractNLI       & 0.293 & \textbf{0.828} \\
\bottomrule
\end{tabular}
\end{table}

\subsection{Complementarity vs. Stochasticity and Compute}
\label{sec:app-equal-compute}

Following from \Cref{sec:ablations}, we test whether stochastic diversity is as effective as retrieval diversity. Instead of ensembling four ICL selection strategies, we ensemble one with four different settings of temperature in the LLM scoring function. Table~\ref{tab:app-equal-compute-ci} shows the MAP for temperature ensembles for each of the four strategies on four different datasets. The default Ens4 outperforms in 12/16 cases with non-overlapping bootstrap 95\% CIs (1000 resamples), while having overlap in the other 4 cases. In no case does temperature ensembling clearly outperform. Table~\ref{tab:app-equal-compute-spread} further shows individual strategy results for each $T$, showing minimal differences in MAP.

To more directly control for test-time compute, we also ensemble four ICL0 calls at temperatures $T \in \{0.3, 0.5, 0.8, 1.0\}$. This matches Ens4's four LLM calls without introducing ICL examples or retrieval diversity. As shown in Table~\ref{tab:app-icl0-equal-compute}, Ens4 significantly outperforms on 6/7 datasets; ContractNLI is statistically tied.

\begin{table*}[t]
\caption{MAP values and 95\% bootstrap CIs for temperature ensembles of single strategy scoring functions. The Ens4 column is reproduced from Table~\ref{tab:main}. The final column indicates whether the CIs overlap, or which method clearly outperforms}
\label{tab:app-equal-compute-ci}
\centering
\small
\begin{tabular}{llrrrl}
\toprule
\textbf{Dataset} & \textbf{Strategy} & \textbf{MAP} & \textbf{95\% CI} & \textbf{Ens4} & \textbf{Overlap} \\
\midrule
PhysioNet & \texttt{anchor\_dpp}  & 0.845 & [0.8223, 0.8678] & 0.879 & No, Ens4 better \\
PhysioNet & \texttt{pattern\_dpp} & 0.842 & [0.8216, 0.8621] & 0.879 & No, Ens4 better \\
PhysioNet & \texttt{bm25}  & 0.855 & [0.8341, 0.8754] & 0.879 & No, Ens4 better \\
PhysioNet & \texttt{random} & 0.882 & [0.8625, 0.8997] & 0.879 & Yes \\
PUMA      & \texttt{anchor\_dpp} & 0.817 & [0.8109, 0.8228] & 0.814 & Yes \\
PUMA      & \texttt{pattern\_dpp}& 0.783 & [0.7762, 0.7898] & 0.814 & No, Ens4 better \\
PUMA      & \texttt{bm25}  & 0.794 & [0.7873, 0.8001] & 0.814 & No, Ens4 better \\
PUMA      & \texttt{random}& 0.804 & [0.7982, 0.8108] & 0.814 & No, Ens4 better \\
SubSumE   & \texttt{anchor\_dpp} & 0.435 & [0.4229, 0.4473] & 0.464 & No, Ens4 better \\
SubSumE   & \texttt{pattern\_dpp} & 0.438 & [0.4251, 0.4505] & 0.464 & No, Ens4 better \\
SubSumE   & \texttt{bm25}  & 0.450 & [0.4370, 0.4624] & 0.464 & No, Ens4 better \\
SubSumE   & \texttt{random} & 0.441 & [0.4283, 0.4541] & 0.464 & No, Ens4 better \\
ECTSum    & \texttt{anchor\_dpp} & 0.518 & [0.5065, 0.5293] & 0.516 & Yes \\
ECTSum    & \texttt{pattern\_dpp} & 0.516 & [0.5047, 0.5268] & 0.516 & Yes \\
ECTSum    & \texttt{bm25} & 0.478 & [0.4667, 0.4901] & 0.516 & No, Ens4 better \\
ECTSum    & \texttt{random} & 0.447 & [0.4367, 0.4582] & 0.516 & No, Ens4 better \\
\bottomrule
\end{tabular}
\end{table*}

\begin{table*}[t]
\caption{MAP of each single-strategy run by temperature $T$. \textbf{span} = max $-$ min across $T$.\label{tab:app-equal-compute-spread}}
\centering
\small
\begin{tabular}{llrrrrr}
\toprule
\textbf{Dataset} & \textbf{ICL Strategy} & $T{=}0.3$ & $T{=}0.5$ & $T{=}0.8$ & $T{=}1.0$ & \textbf{span} \\
\midrule
PhysioNet & \texttt{anchor\_dpp}  & 0.763 & 0.768 & 0.779 & 0.764 & 0.016 \\
PhysioNet & \texttt{pattern\_dpp} & 0.743 & 0.735 & 0.742 & 0.752 & 0.017 \\
PhysioNet & \texttt{bm25}         & 0.769 & 0.760 & 0.762 & 0.769 & 0.009 \\
PhysioNet & \texttt{random}       & 0.798 & 0.792 & 0.789 & 0.798 & 0.009 \\
PUMA      & \texttt{anchor\_dpp}  & 0.777 & 0.771 & 0.772 & 0.768 & 0.009 \\
PUMA      & \texttt{pattern\_dpp} & 0.739 & 0.739 & 0.740 & 0.739 & 0.001 \\
PUMA      & \texttt{bm25}         & 0.739 & 0.740 & 0.739 & 0.732 & 0.008 \\
PUMA      & \texttt{random}       & 0.758 & 0.757 & 0.756 & 0.752 & 0.006 \\
SubSumE   & \texttt{anchor\_dpp}  & 0.339 & 0.346 & 0.350 & 0.352 & 0.013 \\
SubSumE   & \texttt{pattern\_dpp} & 0.344 & 0.356 & 0.358 & 0.351 & 0.014 \\
SubSumE   & \texttt{bm25}         & 0.359 & 0.361 & 0.366 & 0.374 & 0.015 \\
SubSumE   & \texttt{random}       & 0.343 & 0.355 & 0.359 & 0.362 & 0.019 \\
ECTSum    & \texttt{anchor\_dpp}  & 0.449 & 0.443 & 0.432 & 0.423 & 0.026 \\
ECTSum    & \texttt{pattern\_dpp} & 0.442 & 0.444 & 0.426 & 0.425 & 0.019 \\
ECTSum    & \texttt{bm25}         & 0.372 & 0.370 & 0.364 & 0.368 & 0.008 \\
ECTSum    & \texttt{random}       & 0.321 & 0.335 & 0.340 & 0.344 & 0.023 \\
\bottomrule
\end{tabular}
\end{table*}

\begin{table*}[t]
\centering
\caption{Equal-compute comparison between ICL0 and Ens4.
ICL0$\times$4 ensembles four calls at
$T\in\{0.3,0.5,0.8,1.0\}$. 95\% CIs computed via paired bootstrap.}
\label{tab:app-icl0-equal-compute}
\small
\begin{tabular}{lrrrl}
\toprule
Dataset & ICL0 & ICL0$\times$4 & Ens4 & Ens4 $-$ ICL0 [95\% CI] \\
\midrule
ECTSum        & 0.350 & 0.397 & \textbf{0.516} & +0.120 [+0.109, +0.130] \\
Evidence Inf. & 0.206 & 0.274 & \textbf{0.304} & +0.019 [+0.015, +0.023] \\
HotpotQA      & 0.749 & 0.788 & \textbf{0.839} & +0.047 [+0.043, +0.051] \\
PhysioNet     & 0.764 & 0.827 & \textbf{0.879} & +0.034 [+0.018, +0.050] \\
PUMA          & 0.648 & 0.663 & \textbf{0.814} & +0.144 [+0.138, +0.150] \\
SubSumE       & 0.289 & 0.370 & \textbf{0.464} & +0.090 [+0.081, +0.098] \\
ContractNLI   & 0.718 & 0.818 & \textbf{0.828} & +0.010 [$-$0.006, +0.008] \\
\bottomrule
\end{tabular}
\end{table*}

\FloatBarrier

\begin{table*}[t]
\caption{Cross-model validation. Default Ens4 configuration run on Llama3-8B, Qwen3-8B, and GPT-5.6-terra. GPT-5.6-terra results use seeded 500-sample test subsets; the other models use the original test sets. The Gemini-2.5-Flash-Lite rows reproduce entries of Table~\ref{tab:main} for reference. \emph{Best Single} is the post-hoc best individual ICL strategy.}
\label{tab:app-cross-model}
\centering
\footnotesize
\begin{tabular}{p{3.0cm}p{1.6cm}cll cc}
\toprule
Model & Dataset & ICL0 & Best Single (Strategy) & Ens4 & $\Delta$/ICL0 & $\Delta$/Best \\
\midrule
Gemini-2.5-Flash-Lite & PhysioNet & 0.764 & 0.795 (\texttt{random}) & \textbf{0.879} & +0.115 & +0.080 \\
Gemini-2.5-Flash-Lite & PUMA & 0.648 & 0.777 (\texttt{anchor\_dpp}) & \textbf{0.814} & +0.165 & +0.036 \\
Gemini-2.5-Flash-Lite & ECTSum & 0.350 & 0.469 (\texttt{anchor\_dpp}) & \textbf{0.516} & +0.166 & +0.047 \\
Gemini-2.5-Flash-Lite & HotpotQA & 0.749 & 0.749 (\texttt{anchor\_dpp}) & \textbf{0.839} & +0.090 & +0.090 \\
\midrule
Llama3-8B & PhysioNet & 0.488 & 0.405 (\texttt{random}) & \textbf{0.524} & +0.036 & +0.119 \\
Llama3-8B & PUMA & 0.555 & 0.603 (\texttt{random}) & \textbf{0.663} & +0.109 & +0.061 \\
Llama3-8B & ECTSum& 0.220 & 0.178 (\texttt{anchor\_dpp}) & \textbf{0.258} & +0.038 & +0.080 \\
Llama3-8B & HotpotQA  & \textbf{0.372} & 0.235 (\texttt{pattern\_dpp}) & 0.370 & $-$0.002 & +0.135 \\
\midrule
Qwen3-8B& PhysioNet & 0.437 & 0.455 (\texttt{random}) & \textbf{0.600} & +0.163 & +0.145 \\
Qwen3-8B & PUMA & 0.490 & 0.530 (\texttt{anchor\_dpp}) & \textbf{0.597} & +0.106 & +0.066 \\
Qwen3-8B& ECTSum & 0.225 & 0.188 (\texttt{anchor\_dpp}) & \textbf{0.258} & +0.033 & +0.070 \\
Qwen3-8B& HotpotQA  & 0.427 & 0.323 (\texttt{anchor\_dpp}) & \textbf{0.471} & +0.043 & +0.148 \\
\midrule
GPT-5.6-terra & PhysioNet & 0.843 & 0.881 & \textbf{0.921} & +0.078 & +0.040 \\
GPT-5.6-terra & PUMA & 0.712 & 0.838 & \textbf{0.856} & +0.143 & +0.018 \\
GPT-5.6-terra & ECTSum & 0.478 & 0.591 & \textbf{0.612} & +0.134 & +0.021 \\
GPT-5.6-terra & HotpotQA & 0.910 & 0.905 & \textbf{0.923} & +0.014 & +0.018 \\
\bottomrule
\end{tabular}
\end{table*}

\subsection{Cross-Model Validation}
\label{sec:app-cross-model}

\Cref{sec:ablations} introduced cross-model evidence as one of three robustness checks. We run the same untuned Ens4 default configuration on two open-weight LLMs, Llama3-8B and Qwen3-8B, across four representative datasets spanning medical, financial, and general-QA domains: PhysioNet, PUMA, HotpotQA, and ECTSum. We additionally evaluate GPT-5.6-terra on seeded 500-sample test subsets of all seven datasets. Table~\ref{tab:app-cross-model} reports the same four representative datasets for direct cross-model comparison.

\subsection{Task-Hint Format Ablation}
\label{sec:app-task-hint}

The Ens4 default configuration uses a one-line task hint (e.g., \textit{``Task: Identify sentences relevant to the intent.''}) prepended to the prompt. The hint is not part of the ensemble framework---it is an input to every constituent scoring function--- and is therefore optional, but its presence can help the LLM scorer by preparing it for ICL examples.  This section tabulates the importance of the task hint across all 7 datasets and three hint regimes: \emph{short} (the default one-line hint), \emph{full} (a multi-sentence task description), and \emph{no hint} (ICL examples only).

Across the 7 datasets we observe three tiers of sensitivity to the task hint (Table~\ref{tab:hint-ablation}):

\emph{Non-semantic} (PhysioNet): the Ens4 MAP collapses from $0.879$ to $0.549$ without the hint. Personal health information detection cannot be inferred from the contrastive ICL examples alone, because the positive class (personal health information spans) is defined by a regulatory criterion rather than a semantic theme. The scoring LLM is not provided regulatory criteria as part of its context.

\emph{Domain-specific} (Evidence Inference, ECTSum): each loses $\sim$0.04 MAP without the hint, yet the no-hint Ens4 still beats \emph{every} baseline on these datasets, so the hint is helpful but not critical for the Conformal Relevance framework.

\emph{Transparent} (HotpotQA, PUMA, SubSumE, ContractNLI): the relevance criterion is recoverable from ICL examples alone, and the hint's MAP impact is $\leq 0.01$ in either direction. Expanding the hint to a multi-sentence task description (the \emph{full} column) does not consistently help: it improves the MAP on 2/7 datasets and slightly hurts it on 5/7.

\subsection{Stratified ICL Selection Ablation}
\label{sec:app-stratified}

For multi-intent datasets, ICL examples can be drawn from the same intent as the query document $x$ (\emph{stratified}) or from the full pool regardless of intent (\emph{non-stratified}).  Table~\ref{tab:stratified} shows that stratification helps on three of four multi-intent datasets ($+0.018$ to $+0.028$ MAP). The exception is Evidence Inference, where non-stratified selection is marginally better ($-0.006$), possibly because its eight user-defined intents are weakly correlated with relevance patterns.

\subsection{Windowed ICL Selection Ablation} \label{sec:app-windowed-ablation}
\Cref{tab:windowed-ablation} compares MAP with and without windowed ICL on SubSumE across three ICL budgets. Windowed ICL consistently improves MAP; the gain is largest occurs where full-document ICL examples overflow the context window and degrade coherence. At $k = 2$ the improvement is $+0.10$ MAP; at $k = 5$ it is $+0.16$ MAP.

The key design insight is to decouple signal density---which pool examples are selected---from context compression---how they are rendered in the prompt. Windowing compresses rendering without altering selection, so all retrieval strategies benefit.

\FloatBarrier

\begin{table*}[t]
\caption{Task-hint ablation. \textbf{Bold} marks the best Ens4 variant per dataset. ICL0 is shown for reference (uses the same short hint).}
\label{tab:hint-ablation}
\centering
\small
\begin{tabular}{lcccc}
\toprule
Dataset & ICL0 & Ens4 (short) & Ens4 (full) & Ens4 (no hint) \\
\midrule
ECTSum          & 0.350 & \textbf{0.516} & 0.500 & 0.476 \\
Evidence Inf.   & 0.206 & 0.304          & \textbf{0.316} & 0.266 \\
HotpotQA        & 0.749 & 0.839          & 0.832 & \textbf{0.841} \\
PhysioNet       & 0.764 & \textbf{0.879} & 0.872 & 0.549 \\
PUMA            & 0.648 & \textbf{0.814} & 0.806 & 0.804 \\
SubSumE         & 0.289 & \textbf{0.464} & 0.460 & 0.457 \\
ContractNLI     & 0.718 & 0.828          & \textbf{0.829} & 0.821 \\
\bottomrule
\end{tabular}
\end{table*}

\begin{table*}[t]
\caption{Effect of stratified ICL selection on multi-intent datasets. \textbf{Bold} marks the better configuration.}
\label{tab:stratified}
\centering
\small
\begin{tabular}{lcrrc}
\toprule
Dataset & Intents & Stratified & Non-stratified & $\Delta$ \\
\midrule
PUMA              & 5  & \textbf{0.814} & 0.787 & +0.028 \\
SubSumE           & 10 & \textbf{0.464} & 0.444 & +0.020 \\
ContractNLI       & 17 & \textbf{0.828} & 0.811 & +0.018 \\
Evidence Inf.     & 8  & 0.304          & \textbf{0.310} & $-$0.006 \\
\bottomrule
\end{tabular}
\end{table*}

\begin{table*}[t]
  \centering
  \caption{MAP for Windowed vs.\ full ICL on SubSumE (\texttt{random} strategy,
    Gemini-2.5-Flash-Lite). Windowing is beneficial at all ICL budgets; the gap grows with $k$ as full ICL examples become longer.}
  \label{tab:windowed-ablation}
  \small
  \begin{tabular}{@{}lcc@{}}
    \toprule
    ICL budget $k$ & Full ICL & Windowed ICL \\
    \midrule
    2 & 0.206 & 0.309 \\
    3 & 0.174 & 0.294 \\
    5 & 0.112 & 0.267 \\
    \bottomrule
  \end{tabular}
\end{table*}

\FloatBarrier

\section{Empirical Validation of Ensemble Theory}
\label{sec:app-theory-validation}

This appendix expands on the per-sample validation summarized in \Cref{sec:ablations} by checking the condition from Lemma~\ref{lem:complementarity} with bootstrap confidence intervals. We restrict to test samples with more than one positive span.

\paragraph{\Cref{lem:complementarity} condition.}
\Cref{lem:complementarity} states that $\mathrm{Comp}(R_j, R_k) > S_{\max} - S_{\min} \iff  S^{(2)} > S_{\max}$. \Cref{tab:app-prop1-agreement} reports the co-occurrence of these two events per dataset with 95\% bootstrap CIs, using a variant scoring function with $K=2$, as required by the theory. Although the agreement level is high, it is not exactly 100\%. We observed that all $1{,}648$ disagreements across all datasets together can be attributed to floating-point ties where $|S^{(2)} - S_{\max}| < 10^{-12}$ or $\mathrm{Comp} - (S_{\max} - S_{\min}) < 10^{-12}$, which often occurs when $S_{\max} - S_{\min}$ is exactly zero.

\begin{table}[h]
  \centering
  \caption{Co-occurrence of $\mathrm{Comp}(R_j, R_k) > S_{\max} - S_{\min}$ and $S^{(2)} > S_{\max}$. 95\% CIs are bootstrapped with 1000 resamples.}
  \label{tab:app-prop1-agreement}
  \begin{tabular}{lrl}
    \toprule
    Dataset& $n$  & Agreement (95\% CI) \\
    \midrule
    ECTSum        & 11{,}514     & 0.995 \,[0.994, 0.997]\\
    Evidence Inf. & 35{,}580     & 0.998 \,[0.998, 0.999]\\
    HotpotQA      & 59{,}922     & 0.999 \,[0.998, 0.999]\\
    PhysioNet     &  1{,}482     & 0.999 \,[0.998, 1.000]\\
    PUMA          & 27{,}354     & 0.994 \,[0.993, 0.995]\\
    SubSumE       & 10{,}812     & 0.995 \,[0.994, 0.996]\\
    ContractNLI   & 18{,}252     & 0.999 \,[0.999, 1.000]\\
    \midrule
    \textbf{Pooled} & \textbf{471{,}378} & \textbf{0.9955}\\
    \bottomrule
  \end{tabular}
\end{table}

\newpage
\section{Conformal Conciseness Across $\alpha$}
\label{sec:app-alpha-sweep}

\Cref{fig:conciseness} showed the conciseness metric for a range of recall levels $\beta$, and a fixed target coverage $1-\alpha=0.8$. Here we expand those results to additional values of $\alpha$. \Cref{tab:app-alpha-sweep} reports the default Ens4 configuration at $\alpha \in \{0.05, 0.10, 0.20\}$ with $\beta = 0.8$ fixed, while \Cref{fig:app-conciseness-multialpha} shows sweeps over $\beta$ for the smaller $\alpha$ values.  Ens4 improves conciseness over ICL0 on all seven datasets at $\alpha = 0.20$. and onmost datasets at $\alpha = 0.05$ and $\alpha = 0.10$. 

The biggest outlier is Evidence Inference at $\alpha = 0.05$. Evidence Inference's relevance criterion --- whether a stated outcome agrees with experimental evidence --- is non-semantic. We showed in the task hint ablation from App.~\ref{sec:app-task-hint}, \Cref{tab:hint-ablation}, that Evidence Inference is one of two non-semantic datasets that benefits from a lengthier task hint describing what relevance means in that context. \Cref{fig:eviinf-fulltask} shows that including a more descriptive hint (Ens4-FullTask) recovers a positive conciseness gap across all $\beta$.

\begin{table*}[h]
\centering
\caption{Mean conciseness (fraction of irrelevant sentences excluded from the prediction set) for ICL0 and Ens4 across $\alpha \in \{0.05, 0.10, 0.20\}$ at $\beta = 0.8$.}
\label{tab:app-alpha-sweep}
\begin{tabular}{l|cc|cc|cc}
\toprule
& \multicolumn{2}{c|}{$\alpha = 0.05$} & \multicolumn{2}{c|}{$\alpha = 0.10$} & \multicolumn{2}{c}{$\alpha = 0.20$} \\
Dataset & ICL0 & Ens4 & ICL0 & Ens4 & ICL0 & Ens4 \\
\midrule
ContractNLI & 0.097 & 0.335 & 0.204 & 0.606 & 0.849 & 0.940 \\
ECTSum & 0.046 & 0.075 & 0.156 & 0.153 & 0.272 & 0.355 \\
Evidence Inf. & 0.305 & 0.168 & 0.542 & 0.611 & 0.755 & 0.803 \\
HotpotQA & 0.268 & 0.657 & 0.368 & 0.760 & 0.724 & 0.868 \\
PhysioNet & 0.152 & 0.616 & 0.426 & 0.716 & 0.717 & 0.820 \\
PUMA & 0.042 & 0.352 & 0.121 & 0.451 & 0.300 & 0.515 \\
SubSumE & 0.052 & 0.445 & 0.122 & 0.536 & 0.228 & 0.768 \\
\bottomrule
\end{tabular}
\end{table*}

\begin{figure*}[h]
\centering
\begin{subfigure}{0.49\linewidth}
  \centering
  \includegraphics[width=\linewidth, keepaspectratio]{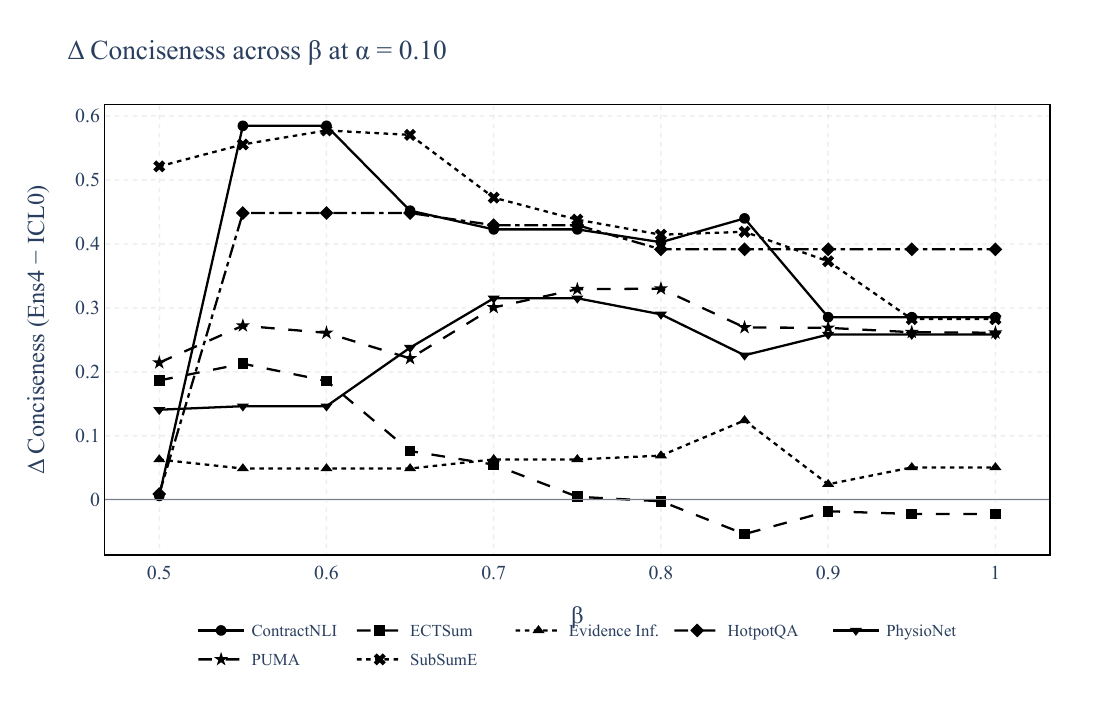}
  \caption{$\alpha = 0.10$.}
  \label{fig:app-conciseness-alpha010}
\end{subfigure}\hfill
\begin{subfigure}{0.49\linewidth}
  \centering
  \includegraphics[width=\linewidth, keepaspectratio]{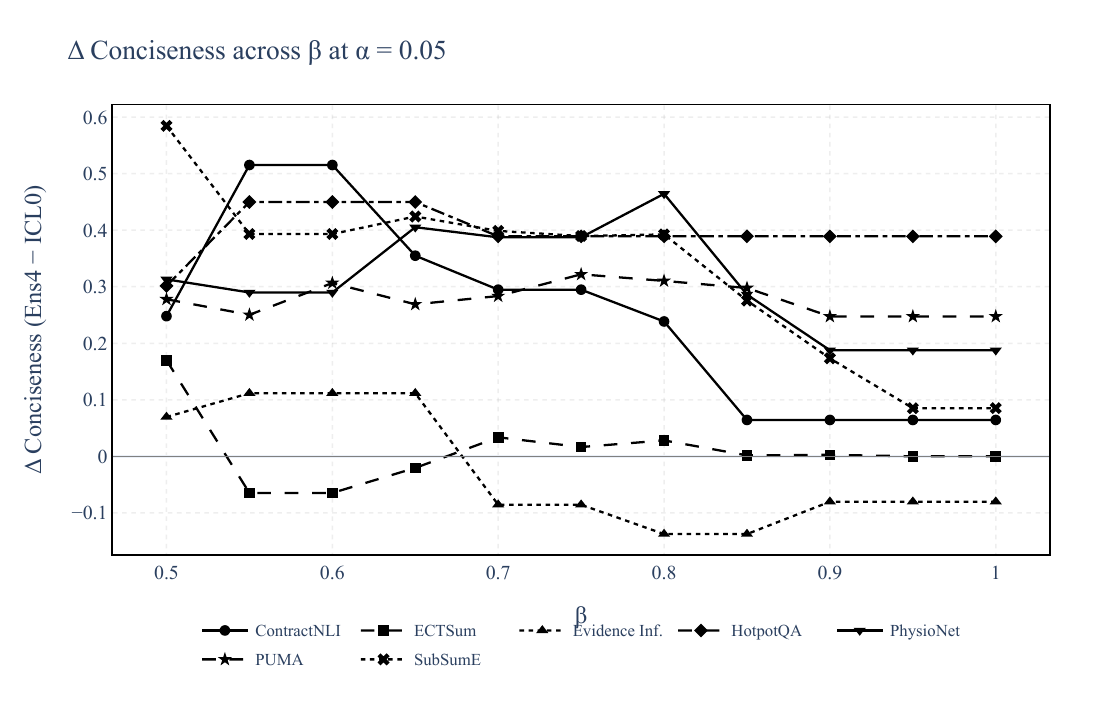}
  \caption{$\alpha = 0.05$.}
  \label{fig:app-conciseness-alpha005}
\end{subfigure}
\caption{Conciseness improvement $\text{Ens4} - \text{ICL0}$ over $\beta$ at stricter coverage targets, complementing Figure~\ref{fig:conciseness}.}
\label{fig:app-conciseness-multialpha}
\end{figure*}

\begin{figure*}[h]
\centering
\includegraphics[width=0.65\linewidth]{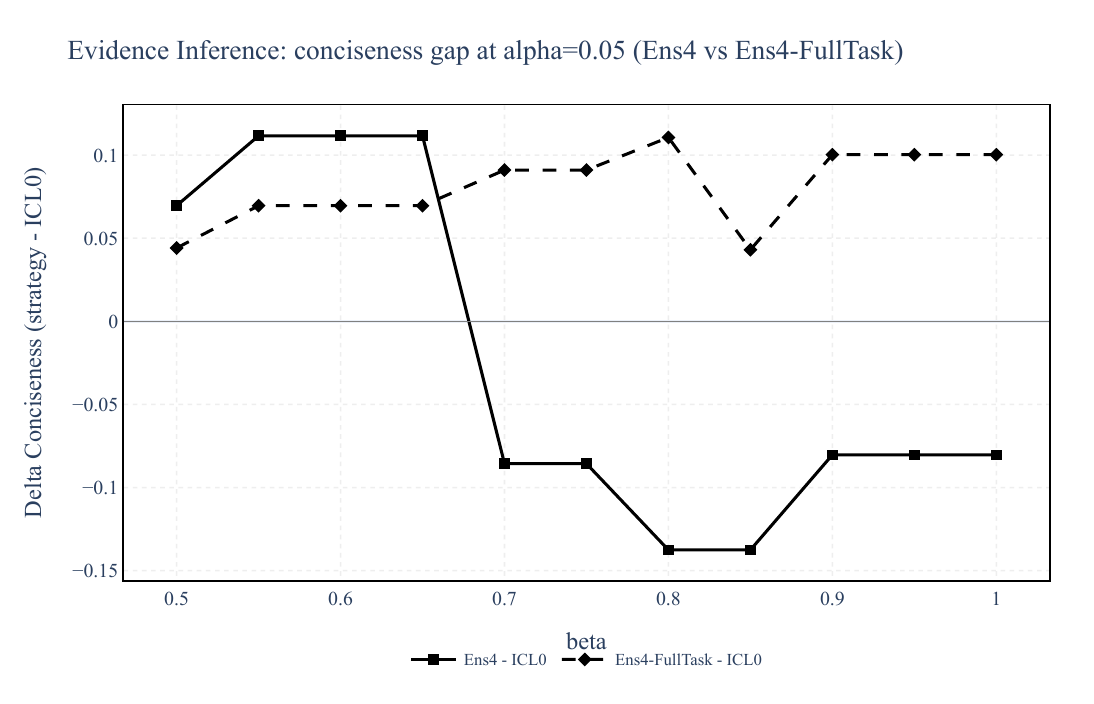}
\caption{Conciseness improvement $\text{Ens4} - \text{ICL0}$ at $\alpha = 0.05$ on Evidence Inference. This dataset benefits from a lengthier manually crafted prompt describing the task, here shown as Ens4-FullTask (dashed). This variant recovers a positive gap across all $\beta$.}
\label{fig:eviinf-fulltask}
\end{figure*}

\begin{table*}[t]
\centering
\caption{Comparison of conformal ensemble rules using the same four
Ens4 constituent scoring functions. Coverage deviation is empirical
minus target coverage, averaged across seven datasets (percentage
points; positive values indicate over-coverage). Conciseness is the
fraction of sentences removed at $(\alpha,\beta)=(0.2,0.8)$; higher
is better.}
\label{tab:conformal-ensemble-rules}
\small
\begin{tabular}{lcccc}
\toprule
& \multicolumn{3}{c}{Coverage deviation (pp)} & \\
\cmidrule(lr){2-4}
Method & $\alpha{=}0.05$ & $\alpha{=}0.10$ & $\alpha{=}0.20$
       & Conciseness \\
\midrule
Majority vote & +1.3 & +0.5 & $-$1.4 & 0.664 \\
Confidence-Level Allocation & +1.7 & +3.6 & +5.5 & 0.426 \\
Union of sets & +4.8 & +9.0 & +15.1 & 0.364 \\
Average single scorer & +0.2 & +0.2 & +0.3 & 0.586 \\
Mean Ensembling (Ens4) & +0.0 & +0.1 & +0.2 & \textbf{0.711} \\
\bottomrule
\end{tabular}
\end{table*}

\subsection{Alternative Conformal Ensemble Rules}
\label{sec:app-conformal-ensembles}

We compare score-level mean ensembling with alternative conformal ensemble rules while keeping the four Ens4 constituent scoring functions fixed. The rules we test are:
\begin{itemize}
    \item \textbf{Majority Vote} \cite{gasparin2024merging} This method uses each scoring function to generate a conformal set, and then merges those sets keeping only elements that appear in the majority. If each individual scorer's set guarantees $1-\alpha$ coverage, the majority vote set only guarantees $1-2\alpha$ coverage, which can lead to undercoverage.
    \item \textbf{COnfidence-Level Allocation} \cite{xu2025aggregating} This method uses each scoring function to generate a conformal set, and then takes the intersection of those sets. Using the same $1-\alpha$ target coverage for each scorer in this way would lead to severe undercoverage. Instead COLA defines a different $\alpha_j$ for each scorer by optimizing them jointly against a set-size metric.
    \item \textbf{Union of Sets} \citep{yang2025selection} This simple method generates a conformal set with each scoring function, and returns the union. The union guarantees $1-\alpha$ coverage, but is typically very conservative and overcovers.
   \item \textbf{Average Single Scorer} This baseline runs each scoring function independently through the standard conformal pipeline, with its own calibrated threshold and prediction sets. We report the average coverage deviation and conciseness across the four scorers and seven datasets. No ensembling occurs; each scorer is evaluated independently, and only the resulting evaluation metrics are averaged.

\end{itemize}

Following \Cref{sec:conformal-results}, we repeat evaluation over 400 random calibration/test splits, with $\beta{=}0.8$ and $\alpha\in\{0.05,0.10,0.20\}$. Table~\ref{tab:conformal-ensemble-rules} reports coverage deviation from the target, averaged across all seven datasets, and conciseness at $(\alpha,\beta)=(0.2,0.8)$.

\end{document}